%% file: main.tex
\documentclass[11pt]{article}

\usepackage{fullpage}
\usepackage[round]{natbib}

\usepackage[utf8]{inputenc} % allow utf-8 input
\usepackage[T1]{fontenc}    % use 8-bit T1 fonts
\usepackage{hyperref}       % hyperlinks
\usepackage{url}            % simple URL typesetting
\usepackage{booktabs}       % professional-quality tables
\usepackage{amsfonts}       % blackboard math symbols
\usepackage{nicefrac}       % compact symbols for 1/2, etc.
\usepackage{microtype}      % microtypography
\usepackage{xcolor}         % colors

\usepackage{microtype}
\usepackage{graphicx}
\usepackage{makecell}
\usepackage{caption}
\usepackage{subcaption}

\usepackage{amsfonts,amsmath,amsthm,amssymb,dsfont,etex}
\allowdisplaybreaks

\usepackage{colortbl}
\definecolor{bgcolor}{rgb}{0.66,0.88,1.00}

\theoremstyle{plain}

\newtheorem{theorem}{Theorem}[section]
\newtheorem{lemma}{Lemma}[section]
\newtheorem{proposition}{Proposition}[section]

\theoremstyle{definition}

\usepackage{xspace}

\xspaceaddexceptions{[]\{\}}

\newcommand{\fedpage}{{\sf\footnotesize FedPAGE}\xspace}
\newcommand{\fedpageNORMAL}{{\sf FedPAGE}\xspace}
\newcommand{\algname}[1]{{\sf\footnotesize#1}\xspace}

\newcommand{\dataset}[1]{{\tt #1}\xspace}

\usepackage{hyperref}

\usepackage{thm-restate}

\usepackage{algorithm}
\usepackage{algorithmicx}
\usepackage{algpseudocode}
\algrenewcommand\algorithmicrequire{\textbf{server input:}}
\algrenewcommand\algorithmicensure{\textbf{client} $i$'s \textbf{input:}}

\algdef{SE}[SUBALG]{Indent}{EndIndent}{}{\algorithmicend\ }%
\algtext*{Indent}
\algtext*{EndIndent}

\newcommand{\squeeze}{\textstyle}

\usepackage{todonotes}

\input{math_commands}

\input{symbols}

\title{\bf {\textsf{FedPAGE}}: A Fast Local Stochastic Gradient Method for \\ Communication-Efficient Federated Learning}

\author{%
Haoyu Zhao \\
Princeton University, USA \\
\texttt{haoyu@princeton.edu} 
\and
Zhize Li\hspace{0.5mm}\thanks{Corresponding author.}\\
KAUST, Saudi Arabia \\
\hspace{-0.5mm}\texttt{zhize.li@kaust.edu.sa} \hspace{-2.5mm}
\and
Peter Richt{\'a}rik \\
KAUST, Saudi Arabia \\
\hspace{-3mm}\texttt{peter.richtarik@kaust.edu.sa}\hspace{-6mm}
}

\date{}

\begin{document}

\maketitle

\begin{abstract}
 Federated Averaging (\algname{FedAvg}, also known as \algname{Local-SGD}) \citep{mcmahan2017communication} is a classical federated learning algorithm in which clients run multiple local \algname{SGD} steps before communicating their update to an orchestrating server.  
 We propose a new federated learning algorithm, \fedpage,  able to further reduce the communication complexity by utilizing the recent optimal \algname{PAGE} method \citep{li2021page} instead of plain \algname{SGD} in \algname{FedAvg}. We show that \fedpage uses much fewer communication rounds than previous local methods for both federated convex and nonconvex optimization.    
 Concretely, 
 1) in the convex setting, the number of communication rounds of \fedpage is $O(\frac{N^{3/4}}{S\epsilon})$, improving the best-known result $O(\frac{N}{S\epsilon})$ of \algname{SCAFFOLD} \citep{karimireddy2020scaffold} by a factor of $N^{1/4}$, where $N$ is the total number of clients (usually is very large in federated learning), $S$ is the sampled subset of clients in each communication round, and $\epsilon$ is the target error;
  2) in the nonconvex setting, the number of communication rounds of \fedpage is $O(\frac{\sqrt{N}+S}{S\epsilon^2})$, improving the best-known result  $O(\frac{N^{2/3}}{S^{2/3}\epsilon^2})$ of \algname{SCAFFOLD} \citep{karimireddy2020scaffold} by a factor of $N^{1/6}S^{1/3}$, if the sampled clients $S\leq \sqrt{N}$.  Note that in both settings, the communication cost for each round is the same for both \fedpage and \algname{SCAFFOLD}. 
  As a result, \fedpage achieves new state-of-the-art results in terms of communication complexity for both federated convex and nonconvex optimization.
\end{abstract}

%\listoftodos

\section{Introduction}
With the rise in the proliferation of mobile and edge devices, and their ever-increasing ability to capture, store and process data, federated learning~\citep{konevcny2016federatedlearning,mcmahan2017communication,kairouz2019advances} has recently emerged as a new machine paradigm for training machine learning models over a vast amount of geographically distributed and heterogeneous devices. Federated learning aims to augment the traditional centralized datacenter focused approach to training machine learning models~\citep{dean2012large,iandola2016firecaffe,goyal2017accurate} with a new decentralized modality that aims to be more energy-efficient, and mainly, more privacy-conscious with respect to the private data stored on these devices.  
In federated learning, the data is stored over a large number of clients, for example, phones, hospitals, or corporations~\citep{konevcny2016federatedoptimization,konevcny2016federatedlearning,mcmahan2017communication,mohri2019agnostic}. Orchestrated by a centralized trusted entity,  these diverse data and compute resources come together to train a single global model to be deployed on all devices.  This is done without the sensitive and private data ever leaving the devices.

\begin{table*}[!t]
	\caption{Number of communication rounds for finding an $\epsilon$-solution of federated convex and nonconvex problems \eqref{eq:problem-setting}, where $\mathbb E f(x) -f^* \le \epsilon$ for convex setting and $\mathbb E \|\nabla f(x)\|_2 \le \epsilon$ for nonconvex setting. In the last column, \emph{$(G,B)$-BGD} means that $\sum_{i=1}^N\|\nabla f_i(x)\|_2^2\le G^2 + B^2\|\nabla f(x)\|_2^2$. $(G,0)$-BGD means $B=0$, and $(0,B)$-BGD means $G=0$. \emph{BV} denotes the ``Bounded Variance'' assumption, i.e., \eqref{eq:ubv-setting}, and \emph{Smooth} stands for standard smoothness assumption, e.g., Assumption~\ref{ass:smooth}. 
	Other notations (i.e., $N,S,M,K$) are summarized in Table \ref{tab:notions}.}
	\label{tab:results}
	\centering
	\begin{tabular}{|c|c|c|c|c|}
		\hline
		Algorithm & Convex setting & Nonconvex setting & Assumption\\
		\hline
		\makecell{\algname{FedAvg} \\ \citep{yu2019parallel}} & --- & $\frac{G^2NK}{\epsilon^2} + \frac{\sigma^2}{NK\epsilon^4}$ & \makecell{Smooth, BV, \\ $(G,0)$-BGD} \\
		\hline
		\makecell{\algname{FedAvg} \\ \citep{karimireddy2020scaffold}} & $\frac{G^2}{S\epsilon^2} + \frac{G}{\epsilon^{3/2}} + \frac{B^2}{\epsilon} + \frac{\sigma^2}{SK\epsilon^2}$ & $\frac{G^2}{S\epsilon^4} + \frac{G}{\epsilon^3} + \frac{B^2}{\epsilon^2} + \frac{\sigma^2}{SK\epsilon^4}$  & \makecell{Smooth, BV, \\ $(G,B)$-BGD} \\
		\hline
		\makecell{\algname{FedProx} \\ \citep{sahu2018convergence}} & $\frac{B^2}{\epsilon}$ & --- & \makecell{Smooth, $S = N$, \\ $(0,B)$-BGD} \\
		\hline
		\makecell{\algname{VRL-SGD} \\ \citep{liang2019variance}} & --- & $\frac{N}{\epsilon^2} + \frac{N\sigma^2}{K\epsilon^4}$ & \makecell{Smooth, BV, \\ $S = N$} \\
		\hline
		\makecell{\algname{S-Local-SVRG}\\ \citep{gorbunov2020local}} & $\frac{M^{1/3}/K^{1/3} + \sqrt{M/NK^{2}}}{\epsilon} \footnotemark$ & --- & \makecell{Smooth, (BV), \\ $S = N,~ K\leq M$} \\
		\hline
		\makecell{\algname{SCAFFOLD}\\ \citep{karimireddy2020scaffold}} & $\frac{N}{S\epsilon} + \frac{\sigma^2}{SK\epsilon^2}$ & $\frac{N^{2/3}}{S^{2/3}\epsilon^2} + \frac{\sigma^2}{SK\epsilon^4}$ & Smooth, BV \\
		\hline
		\rowcolor{bgcolor}
		\gape{\makecell{\fedpage \\ (this paper) }}& $\frac{N^{3/4}}{S\epsilon}$ \footnotemark  & $\frac{N^{1/2}+S}{S\epsilon^2}$ & Smooth, (BV) \footnotemark \\
		\hline
	\end{tabular}
\end{table*} 
\footnotetext[1]{We point out that \algname{S-Local-SVRG} \citep{gorbunov2020local} only considered the case where $S = N$ and $K\leq M$, i.e., the sampled clients $S$ always is the whole set of clients $N$ for all communication rounds. As a result, the total communication complexity (i.e., number of rounds $\times$ communicated clients $S$ in each round) of \algname{S-Local-SVRG} is $O\Big(\frac{NM^{1/3}/K^{1/3} + \sqrt{NM/K^{2}}}{\epsilon}\Big)$ (note that here $K\leq M$), which is worse than $O(\frac{N}{\epsilon})$ of  \algname{SCAFFOLD} \citep{karimireddy2020scaffold} and $O(\frac{N^{3/4}}{\epsilon})$ of our \fedpage.}
\footnotetext[2]{In the convex setting, we state the result of \fedpage in the case $S\le\sqrt{N}$ (typical in practice) for simple presentation, where $N$ is the total number of clients and $S$ is the number of sampled subset clients in each communication round (see Table \ref{tab:notions}).  Please see Theorem~\ref{thm:convex} for other results of \fedpage in the cases $S>\sqrt{N}$.}

\footnotetext{\fedpage also works under the BV assumption by using moderate minibatches for local clients, and more importantly the number of communication rounds of \fedpage still remains the same as in the last row of Table \ref{tab:results} for both convex (see Theorem~\ref{thm:convex}) and nonconvex (see Theorem~\ref{thm:non-convex}) settings.}

One of the key challenges in federated learning comes from the fact that communication over a heterogeneous network is extremely slow, which leads to significant slowdowns in training time. While a centralized model may train in a matter of hours or days, a comparable federated learning model may require days or weeks for the same task.  For this reason, it is imperative that in the design of federated learning algorithms one focuses special attention on the communication bottleneck, and designs communication-efficient learning protocols capable of producing a good model.

There are two popular lines of work for tackling this communication-efficient federated learning problem. The first makes use of general and also bespoke {\em lossy compression operators} to compress the communicated messages before they are sent over the network~\citep{mishchenko2019distributed, li2020unified, li2020acceleration, gorbunov2021marina, li2021canita}, and the second line bets on increasing the local workload by performing {\em multiple local update steps}, e.g., multiple \algname{SGD} steps, before communicating with the orchestrating server~\citep{localSGD-Stich, Blake2020, gorbunov2020local, karimireddy2020scaffold}. 

In this paper, we focus on the latter approach (multiple local update steps in each round) to alleviating the communication bottleneck in federated learning. One of the earliest and classical methods proposed in this context is \algname{FedAvg/local-SGD}~\citep{mcmahan2017communication, sahu2018convergence, yu2019parallel, li2019convergence, haddadpour2019convergence, localSGD-Stich, gorbunov2020local}. However, the method has remained a heuristic until recently, even in its simplest form as local gradient descent, particularly in the important heterogeneous data regime~\citep{khaled2020tighter, Blake2020}. Further improvements on vanilla \algname{Local-SGD} have been proposed, leading to methods such as \algname{Local-SVRG}~\citep{gorbunov2020local} and \algname{SCAFFOLD}~\citep{karimireddy2020scaffold}. In particular, \cite{gorbunov2020local} also provide a unified framework for the analysis of many local methods in the strongly convex and convex settings.

\subsection{Our contributions}

Although there are many works on local gradient-type methods, the communication complexity in existing works on local methods is still far from optimal. In this paper, we introduce a new local method \fedpage, significantly improving the best-known results for both federated convex and nonconvex settings (see Table \ref{tab:results}).
Now, we summarize our main contributions as follows:
\begin{enumerate}
    \item We develop and analyze, \fedpage, a fast local method for communication-efficient federated learning. \fedpage can be loosely seen as a local/federated version of the \algname{PAGE} algorithm of \cite{li2021page}, which is a recently proposed optimal optimization method for solving smooth nonconvex problems. In particular, for the nonconvex setting, \fedpage with local steps $K=1$ reduces to the original \algname{PAGE} algorithm. Our general analysis of \fedpage with $K\geq 1$ also recovers the optimal results of \algname{PAGE} (see Theorem~\ref{thm:page-non-convex} and \ref{thm:non-convex}), thus \fedpage  substantially improves the best-known non-optimal result of \algname{SCAFFOLD} \citep{karimireddy2020scaffold} by a factor of $N^{1/6}S^{1/3}$ (see Table \ref{tab:results}).
     
     \item For the convex setting, we present the convergence theorems for \fedpage with local steps $K=1$ (standard setting, Theorem~\ref{thm:page-convex}) and $K\geq 1$ (multiple local update setting, Theorem~\ref{thm:convex}).
     Moreover, \fedpage also improves the best-known result of \algname{SCAFFOLD} \citep{karimireddy2020scaffold} by a large factor of $N^{1/4}$ (see Table \ref{tab:results}).
     
     \item Finally, we first conduct the numerical experiments for showing the effectiveness of multiple local update steps (see Section \ref{sec:exp-local}). The experiments indeed demonstrate that \fedpage with multiple local steps $K\geq 1$ is better than that with $K=1$ (no multiple local updates).
     Then we also conduct experiments for comparing the performance of different local methods such as \algname{FedAvg} \citep{mcmahan2017communication},  \algname{SCAFFOLD} \citep{karimireddy2020scaffold} and our \fedpage (see Section \ref{sec:exp-comparison}). The experiments show that \fedpage always converges much faster than \algname{FedAvg}, and at least as fast as \algname{SCAFFOLD} (usually much better than \algname{SCAFFOLD}), confirming the practical superiority of \fedpage.
\end{enumerate}

\subsection{Related works} Optimization algorithms for federated learning have a close relationship with the algorithms designed for standard finite-sum problem $\min_x \frac{1}{n}\sum_{i=1}^n f_i(x)$. In the federated learning setting, we can think of the loss function of the data on a single client as function $f_i$, and the optimization problem becomes a finite-sum problem. The \algname{SGD} is perhaps the most famous algorithm for solving the finite-sum problem, and in one variant or another, it is widely applied in training deep neural networks. However, the convergence rates of plain \algname{SGD} in the convex and nonconvex settings are not optimal.  This motivated  a feverish research activity in the optimization and machine learning communities over the last decade, and these efforts led to theoretically and practically improved variants of \algname{SGD}, such as \algname{SVRG}, \algname{SAGA}, \algname{SARAH}, \algname{SPIDER}, and \algname{PAGE}~\citep{johnson2013accelerating, defazio2014saga, nguyen2017sarah, fang2018spider, li2021page} and many of their variants possibly with acceleration/momentum \citep{allen2017katyusha, lan2018random, lei2017non, li2018simple, zhou2018stochastic, wang2018spiderboost, kovalev2019don, ge2019stable, li2019ssrgd, lan2019unified, li2020fast, li2021anita}.

However, the above well-studied finite-sum problem is not equivalent to the federated learning problem \eqref{eq:problem-setting} as one needs to account for the communication, which forms the main bottleneck. As we discussed before, there are at last two sets of ideas for solving this problem: communication compression, and local computation.
There are lots of works belonging to these two categories.
In particular, for the first category, the current state-of-the-art results in strongly convex, convex, and nonconvex settings are given by \cite{li2020acceleration, li2021canita, gorbunov2021marina}, respectively. For the second category, local methods such as \algname{FedAvg} \citep{mcmahan2017communication} and \algname{SCAFFOLD} \citep{karimireddy2020scaffold} perform multiple local update steps in each communication round in the hope that these are useful to decrease the number of communication rounds needed to train the model.
In this paper, we provide new state-of-the-art results of local methods for both federated convex and nonconvex settings, which significantly improves the previous best-known results of \algname{SCAFFOLD} \citep{karimireddy2020scaffold} (See Table \ref{tab:results}).

\section{Setup and Notation}

\begin{table}[t]
	\centering	
	\caption{Summary of notation used in this paper}
	\label{tab:notions}
	\begin{tabular}{| c | l |}
		\hline
		$N, S, i$ & total number, sampled number, and index of clients \\
		$M$ & total number of data in each client\\
		$R, r$ & total number and index of communication rounds \\
		$K, k$ & total number and index of local update steps \\
		$x^r$ & model parameters before round $r$ \\
		$g^r$ & server update within round $r$ \\
		$y_{i,k}^r$ & $i$-th client's model in round $r$ before local step $k$ \\
		$g_{i,k}^r$ & $i$-th client's update in round $r$ within local step $k$  \\
		$\nabla_{\gI} f_i(x)$ & estimator of $\nabla f_i(x)$ using a sampled minibatch $\gI$ \\
		& $\nabla_{\gI} f_i(x) = 1 / |\gI| \sum_{j\in\gI}\nabla f_{i,j}(x)$ \\
		\hline
	\end{tabular}

\end{table}

We formalize the problem as minimizing a finite-sum functions:
\begin{equation}
    \squeeze \min \limits_{x\in\mathbb R^d}\left\{f(x) := \frac{1}{N}\sum \limits_{i=1}^N f_i(x)\right\}, \text{where } f_i(x) := \frac{1}{M}\sum_{i=1}^M f_{i,j}(x).\label{eq:problem-setting}
\end{equation}
In this formulation, each function $f_i(\cdot)$ stands for the loss function with respect to the data stored on client/device/machine $i$, and each function $f_{i,j}(\cdot)$ stands for the loss function with respect to the $j$-th data on client $i$. Besides, we assume that the minimum of $f$ exists, and we use $f^*$ and $x^*$ to denote the minimum of function $f$ and the optimal point respectively.

We will use $[n]$ to denote the set $\{1,2,\dots,n\}$, $\|\cdot\|$ to denote the Euclidean norm for a vector, and $\langle u,v\rangle$ to denote the inner product of vectors $u$ and $v$. We use $O(\cdot)$ and $\Omega(\cdot)$ to hide the absolute constants.

In this paper, we consider two cases: nonconvex case and convex case. In the nonconvex case, each individual function $f_i$ and the average function $f$ can be nonconvex, and we assume that the functions $\{f_{i,j}\}_{i\in [N], j\in [M]}$ are $L$-smooth.

\begin{restatable}[$L$-smoothness]{assumption}{asssmooth}\label{ass:smooth}
All functions $f_{i,j}:\mathbb R^d\to \mathbb R$ for all $i\in[N], j\in [M]$ are $L$-smooth. That is,  there exists $L \ge 0$ such that for all $x_1,x_2\in\mathbb R^d$ and all $i\in [N], j\in [M]$,
    \[\|\nabla f_{i,j}(x_1) - \nabla f_{i,j}(x_2)\| \le L\|x_1-x_2\|.\]
\end{restatable}

If the functions $\{f_{i,j}\}_{i\in [N], j\in [M]}$ are $L$-smooth, we can conclude that functions $\{f_i\}$ are also $L$-smooth and function $f$ is $L$-smooth.
In this nonconvex setting, the optimization algorithm aims to find a point such that the expectation of the gradient norm is small enough: $\mathbb E\|\nabla f(x)\|_2 \le \epsilon$.

Then, in the convex case, each \emph{individual function} $f_i$ can be \emph{nonconvex}, but we require that the \emph{average function} $f$ to be \emph{convex}. We also assume that the functions $\{f_{i,j}\}$ are $L$-smooth (Assumption~\ref{ass:smooth}). Under this convex setting, the algorithm will find a point such that the expectation of the function value is close to the minimum: $\mathbb E f(x) - f^* \le \epsilon$.

In Section \ref{sec:nonconvex-1} and Section \ref{sec:convex-1}, we will also discuss and analyze a special case (i.e., the local steps $K=1$) of our \fedpage algorithm. When we discuss the special case under the nonconvex and convex setting, we do not need all of the functions $\{f_{i}\}$ to be $L$-smooth. Instead, we only need the following average $L$-smoothness assumption, which is a weaker assumption compared with the smoothness Assumption~\ref{ass:smooth}.

\begin{restatable}[Average $L$-smoothness]{assumption}{assavgsmooth}\label{ass:avgsmooth}
    A function $f:\mathbb R^d \to \mathbb R$ is average $L$-smooth if there exists $L \ge 0$ such that for all $x_1,x_2\in\mathbb R^d$,
    \[\mathbb E_i\|\nabla f_i(x_1) - \nabla f_i(x_2)\|^2 \le L^2\|x_1-x_2\|^2.\]
\end{restatable}

If the functions $\{f_i\}$ are average $L$-smooth (Assumption~\ref{ass:avgsmooth}), then $f(x) =\frac{1}{N}\sum_{i=1}^N f_i(x)$ is also $L$-smooth, i.e., for all $x_1,x_2\in\mathbb R^d$, 
$\|\nabla f(x_1)-\nabla f(x_2)\|_2\le L\|x_1-x_2\|_2.$

If the number of data on a single client is very large and one cannot compute the local full gradients of clients, one needs the following assumption in which the gradient variance on each client is bounded.  

\begin{restatable}[Bounded Variance]{assumption}{assvariance}\label{ass:bounded-variance}
    There exists $\sigma \geq 0$ such that for any client $i\in [N]$ and  $x \in \R^d$,
    \begin{equation}
        \frac{1}{M}\sum_{j=1}^M \|\nabla f_{i,j}(x) - \nabla f_i(x)\|_2^2 \le \sigma^2. \label{eq:ubv-setting}
    \end{equation}
\end{restatable}

\section{The \fedpageNORMAL \ Algorithm}\label{sec:alg}
In this section, we introduce our \fedpage\ algorithm. To some extent, our \fedpage\ algorithm is the local version of \algname{PAGE}~\citep{li2021page}: when the clients communicate with the server, \fedpage\ behaves similar to \algname{PAGE}, and when the clients update the model locally, each client updates several steps. If we set the number of local updates to one, \fedpage\ reduces to the original \algname{PAGE} algorithm.

\begin{algorithm}[!t]
	\caption{\fedpageNORMAL}
	\label{alg:fedpage}
	\begin{algorithmic}[1]
		\Require initial point $x^0$, global step size $\eta_g$, probabilities $\{p_r\}$, sampled clients size $S$
		\Ensure local step size $\eta_{l}$, minibatch sizes $b_{1},b_2, b_3$
		\For{$r = 0,1,2,\dots,R$}
		\State sample $q\sim \text{Bernoulli}(p_r)$
		\If{$q = 1$} \label{line:3}
		\State clients $S^r = [N]$, communicate $x^r$ to all $i\in S^r$ \label{line:4}
		\State \textbf{on client} $i\in S^r$ \textbf{in parallel do}
		\Indent
		\State uniformly sample minibatch $\gI_1\subset [M]$ with size $b_{1}$
		\State compute the gradient estimator $g_i^r \leftarrow \nabla_{\gI_1} f_i(x^r)$
		\EndIndent
		\State \textbf{end on client}
		\State $g^r \leftarrow \frac{1}{N}\sum_{i\in [N]}g_i^r$ \label{algline:full-grad-est}
		\Else \label{line:10}
		\State sample clients $S^r\subseteq [N]$ with size $S$, communicate $(x^r,x^{r-1},g^{r-1})$ to all $i\in S^r$
		\State \textbf{on client} $i\in S^r$ \textbf{in parallel do}
		\Indent
		\State $y_{i,0}^r \leftarrow x^r$
		\State uniformly sample minibatch $\gI_2\subset [M]$ with size $b_2$
		\State $g_{i,0}^r \leftarrow \nabla_{\gI_2} f_i(x^r) - \nabla_{\gI_2} f_i(x^{r-1}) + g^{r-1}$ \label{algline:local-full-grad-est}
		\State $y_{i,1}^r \leftarrow y_{i,0}^r - \eta_{l}g_{i,0}^r$
		\For{$k=1,2,\dots,K-1$} \label{line:17}
		\State uniformly sample minibatch $\gI_3\subset [M]$ with size $b_3$
		\State $g_{i,k}^r \leftarrow \nabla_{\gI_3} f_{i}(y_{i,k}^r) - \nabla_{\gI_3} f_{i}(y_{i,k-1}^r) + g_{i,k-1}^r$ \label{algline:local-est}
		\State $y_{i,k+1}^r \leftarrow y_{i,k}^r - \eta_{l}g_{i,k}^r$
		\EndFor  \label{line:21}
		\State $\Delta y_i^r \leftarrow  x^r - y_{i,K}^r$
		\EndIndent
		\State \textbf{end on client}
		\State $g^r\leftarrow\frac{1}{K\eta_l S}\sum_{i\in S^r}\Delta y_i^r$ \label{line:24}
		\EndIf \label{line:25}
		\State $x^{r+1} \leftarrow x^{r}-\eta_g g^r$ \label{line:26}
		\EndFor
	\end{algorithmic}
\end{algorithm}

Our \fedpage algorithm is given in Algorithm~\ref{alg:fedpage}. There are two cases in each round $r$: 1) with probability $p_r$ (typically very small), the server communicates with all clients in order to get a more accurate gradient of function $f$ (Line \ref{line:3}--\ref{algline:full-grad-est});
2) with probability $1-p_r$, the server communicates with a subset of clients with size $S$ and the local clients perform $K$ local steps (Line \ref{line:10}--\ref{line:24}).

For Case 1), the server broadcasts the current model parameters $x^r$ to all of the clients. Then, each client computes the local gradient estimator $\nabla_{\gI_1} f_i(x^r)$ of the gradient $\nabla f_i(x^r)$ and sends back to the server. The local gradient estimator takes $b_1$ minibatch samples ($|\gI_1| = b_1$)  to estimate the gradient of $f_i$ and different clients sample different sets $\gI_1$. Here, we want $\nabla_{\gI_1} f_i(x^r)$ to be as closed to $\nabla f_i(x^r)$ as possible, choosing a moderate size $b_1$ usually is enough. The server average all of the gradient and get the averaged gradient $g^r = \frac{1}{N}\sum_{i\in [N]} \nabla_{\gI_1} f_i(x^r)$ and takes a step with global step size $\eta_g$ (see Line \ref{algline:full-grad-est} and Line \ref{line:26}).

For Case 2), the server first broadcasts $(x^r,x^{r-1},g^{r-1})$ to the sampled subset clients $S^r$, and the clients initialize $y_{i,0}^r \leftarrow x^r$.
Here, $y_{i,k}^r$ is $i$-th client's model in round $r$ before local step $k$, and $g_{i,k}^r$ denotes $i$-th client's gradient estimator for step $k$ in round $r$. Then for the first local step of client $i$, the local gradient estimator is computed in Line \ref{algline:local-full-grad-est} as
\[g_{i,0}^r \leftarrow \nabla_{\gI_2} f_i(x^r) - \nabla_{\gI_2} f_i(x^{r-1}) + g^{r-1},\]
where $\nabla_{\gI_2} f_i(\cdot)$ is the gradient estimator of $\nabla f_i(\cdot)$ with minibatch size $b_2$. Here, we also want $\nabla_{\gI_2} f_i(\cdot)$ to be as closed to $\nabla f_i(\cdot)$ as possible, similarly choosing a moderate size $b_2$ usually is enough. This update rule is similar to \algname{PAGE} \citep{li2021page} and in particular if the local steps $K=1$, our \fedpage algorithm reduces to \algname{PAGE}.

For client $i$'s local step $k$ such that $1 \le k\le K-1$, the local gradient estimator is computed in Line  \ref{algline:local-est} as
\[g_{i,k}^r \leftarrow \nabla_{\gI_3} f_{i}(y_{i,k}^r) - \nabla_{\gI_3} f_{i}(y_{i,k-1}^r) + g_{i,k-1}^r.\] 
Here $\gI_3$ is a minibatch of functions with size $b_3$ that we used to compute the gradient estimator $g_{i,k}^r$. Different from the previous gradient estimators using minibatches with size $b_1$ and $b_2$, here we want $b_3$ to be small enough to reduce the computation cost as there are $K$ local steps (Line \ref{line:17}--\ref{line:21}). 
In particular, we can choose $b_3 = 1$, i.e., just sample an index $j$ from $[M]$ and the gradient estimator becomes
\[g_{i,k}^r \leftarrow (\nabla f_{i,j}(y_{i,k}^r) - \nabla f_{i,j}(y_{i,k-1}^r)) + g_{i,k-1}^r.\]

The local model update is given by $y_{i,k+1}^r = y_{i,k}^r - \eta_{l}g_{i,k}^r$ where $\eta_{l}$ is the local step size. 
After $K$ local steps, client $i$ computes the local changes $\Delta y_i^r = y_{i,K}^r - x^r$ within round $r$ and sends back to the server. After receiving the local changes $\Delta y_i^r$ for the selected clients $i\in S^r$, the server computes the average gradient estimator on these selected clients in Line \ref{line:24} as 
\begin{align*}
    g^r =& \squeeze \frac{1}{K\eta_l S}\sum \limits_{i\in S^r}\Delta y_i^r.
\end{align*}

After obtaining the gradient estimator $g^r$ (in Line \ref{algline:full-grad-est} or \ref{line:24}), the server updates the model using a global step size $\eta_g$ in Line \ref{line:26} as 
 \[x^{r+1} = \squeeze x^{r}-\eta_g g^r.\]

The intuition of \fedpage\ works as follow: when the local step size $\eta_l$ is not too large, we can expect that the local model updates are close to the original model, that is 
$y_{i,k}^r\approx x^r, \forall k < K$,
and the gradient is also close to each other, $\nabla f_{i,j}(y_{i,k}^r) \approx \nabla f_{i,j}(x^r), \forall k < K$.
Then each local gradient estimator $g_{i,k}^r$ is close to
\[\squeeze g_{i,k}^r = (\nabla f_{i,j}(y_{i,k}^r) - \nabla f_{i,j}(y_{i,k-1}^r)) + g_{i,k-1}^r \approx \nabla f_{i,j}(x^r) - \nabla f_{i,j}(x^{r}) + g_{i,k-1}^r = g_{i,k-1}^r = g_{i,0}^r,\]
and the aggregated global gradient estimator $g^r$ is close to
\[\squeeze g^r \approx \frac{1}{S}\sum \limits_{i\in S^r}\left(\nabla f_i(x^r) - \nabla f_i(x^{r-1}) + g^{r-1}\right).\]
This biased recursive gradient estimator $g^r$ is similar to the gradient estimator in \algname{SARAH} \citep{nguyen2017sarah} or \algname{PAGE}~\citep{li2021page}, and thus the performance of \fedpage in terms of communication rounds should be similar to the optimal convergence results of \algname{PAGE} \citep{li2021page}.

\section{\fedpageNORMAL in Nonconvex Setting}\label{sec:nonconvex}
In this section, we show the convergence rate of \fedpage in the nonconvex setting. 
As we discussed before, if the number of local steps in \fedpage is set to $K=1$, \fedpage reduces to the original \algname{PAGE} algorithm \citep{li2021page}. In Section \ref{sec:nonconvex-1}, we first review the optimal convergence result of \algname{PAGE} in the nonconvex setting~\citep{li2021page}. In Section \ref{sec:nonconvex-general}, we show our convergence result for general local steps $K\geq 1$ in the nonconvex setting.

\subsection{\fedpageNORMAL with local step $K=1$}\label{sec:nonconvex-1}
In this section, we review the convergence rate of \algname{PAGE} in the nonconvex setting. The following theorem is directly derived from Theorem 1 in \citep{li2021page, li2021short}.

\begin{theorem}[Theorem 1 in \citep{li2021page}]\label{thm:page-non-convex}
    Suppose that Assumption~\ref{ass:avgsmooth} holds, i.e. $\{f_i\}$ are average $L$-smooth. If we choose the sampling probability $p_0 = 1$ and $p_r \equiv p = \frac{S}{N}$ for every $r \ge 1$, the global step size 
    \[\squeeze \eta_g = \frac{1}{L(1+\sqrt{\frac{1-p}{2pS}})},\]
    then \fedpage with $K=1$ (\algname{PAGE}) will find a point $x$ such that $\mathbb E \|\nabla f(x)\|_2 \le \epsilon$ with the number of communication rounds bounded by
    \[ \squeeze R = O\left(\frac{L(\sqrt{N}+S)}{S\epsilon^2}\right).\]
\end{theorem}

\cite{li2021page} also provide the tight lower bound (Theorem 2 of \citep{li2021page}) indicating that the convergence result of \algname{PAGE} (i.e., Theorem~\ref{thm:page-non-convex}) is optimal in this nonconvex setting.

\subsection{\fedpageNORMAL with general local steps $K \ge 1$}\label{sec:nonconvex-general}
In this section, we provide the general result of our \fedpage with any local steps $K \geq 1$ in the nonconvex setting. Here we assume that the functions $\{f_{i,j}\}$ are $L$-smooth (Assumption~\ref{ass:smooth}), and we obtain the following theorem.

\begin{restatable}[Convergence of \fedpage\ in nonconvex setting]{theorem}{thmnonconvex}\label{thm:non-convex}
    Under Assumption~\ref{ass:smooth} (and Assumption \ref{ass:bounded-variance}), if we choose the sampling probability $ p_r \equiv p = \frac{S}{N}$ for every $r \ge 1$ and $p_0 = 1$, the minibatch sizes $b_1 = \min\{M, \frac{24\sigma^2}{S\epsilon^2}\}, b_2 = \min\{M, \frac{48\sigma^2}{pS\epsilon^2}\}$, and the global and local step sizes 
    \[\eta_g \le \frac{1}{L\left(1+\sqrt{\frac{3(1- p/3)}{2pS}}\right)}, \qquad \eta_l \le \frac{\sqrt{2}p}{24\sqrt{S}KL},\]
    then \fedpage will find a point $x$ such that $\mathbb E \|\nabla f(x)\|_2 \le \epsilon$ within the following number of communication rounds:
    \[R = O\left(\frac{L(\sqrt{N}+S)}{S\epsilon^2}\right).\]
\end{restatable}

Now we compare the communication cost of \fedpage (Theorem~\ref{thm:non-convex}) with previous state-of-the-art  \algname{SCAFFOLD} \citep{karimireddy2020scaffold}. 
The number of communication round for \algname{SCAFFOLD} to find a point $x$ such that $\mathbb E \|\nabla f(x)\|_2 \le \epsilon$ (the original \algname{SCAFFOLD}~\citep{karimireddy2020scaffold} uses $\mathbb E\|\nabla f(x)\|_2^2\le \epsilon$) is bounded by
\begin{align}
	R_{\algname{SCAFFOLD}} = O\left(\left(\frac{N}{S}\right)^{2/3}\frac{L}{\epsilon^2}\right). \label{eq:scaffoldR}
\end{align}

Beyond the number of communication rounds in \fedpage and \algname{SCAFFOLD}, we also need to compare the communication cost during each round (i.e., number of clients communicated with the server in the round). 
For our \fedpage, in each round, with probability $p = \frac{S}{N}$, the server communicates with all clients $N$, and with probability $1-p$, the server communicates with a sampled subset clients with size $S$, and the communicated clients within each round is $\frac{S}{N}\times N+(1-\frac{S}{N})\times S < 2S$ in expectation.
For \algname{SCAFFOLD}, in each round, the server communicates with $S$ sampled clients.
Thus the communication cost for each round is the same $O(S)$ for both \fedpage and \algname{SCAFFOLD}.
As a result, to compare the communication complexity of \fedpage and \algname{SCAFFOLD}, it is equivalent to compare the number of communication rounds. 
According to \eqref{eq:scaffoldR} and Theorem~\ref{thm:non-convex} (e.g., with sampled clients $S\leq \sqrt{N}$), then the communication rounds of \fedpage is smaller than previous state-of-the-art \algname{SCAFFOLD} by a factor of $N^{1/6}S^{1/3}$. Also note that the number of clients $N$ is usually very large in the federated learning problems.

\section{\fedpageNORMAL in Convex Setting}\label{sec:convex}
In this section, we show the convergence results of \fedpage\ in the convex setting.  Here the algorithms aim to find a point $x$ such that $\mathbb E f(x) - f^* \le \epsilon$ for convex case instead of $\mathbb E \|\nabla f(x)\|_2 \le \epsilon$ for nonconvex case. 
Similar to the nonconvex setting of Section \ref{sec:nonconvex}, we first show the convergence result when $K=1$ in Section \ref{sec:convex-1}, where our \fedpage algorithm reduces to \algname{PAGE}, and then in Section \ref{sec:convex-general} we show the general result of \fedpage with any local steps $K\geq 1$. 
We would like to point out that, in the original \algname{PAGE} paper \citep{li2021page}, there is no result in the convex setting, and our result in Section \ref{sec:convex-1} fills this blank for \algname{PAGE}.

\subsection{\fedpageNORMAL with local step $K=1$}\label{sec:convex-1}
Now we show the convergence result of \fedpage with $K=1$ (\algname{PAGE}) in the convex setting. We assume that the functions $\{f_i\}$ are average $L$-smooth and function $f = \frac{1}{N}\sum_{i=1}^nf_i(x)$ is convex.
\begin{restatable}[Convergence of \fedpage in convex setting when $K=1$]{theorem}{thmpageconvex}\label{thm:page-convex}
    Suppose that $f$ is convex and Assumption~\ref{ass:avgsmooth} holds, i.e. $\{f_i\}$ are average $L$-smooth. If we choose the sampling probability $p_0 = 1$ and $p_r \equiv p = \frac{S}{N}$ for every $r \ge 1$, the number of local steps $K=1$, the minibatch sizes $b_1 = b_2 = M$, the global step size
    \[\squeeze \eta_g \le \Theta\left(\frac{(S+N^{3/4})\frac{S}{N}}{L(S+\sqrt{N})}\right),\]
    then \fedpage\ will find a point $x$ such that $\mathbb E f(x) - f^* \le \epsilon$ with the number of communication rounds bounded by
    \[\squeeze R = \left\{\begin{aligned}
	& O\left(\frac{N^{3/4}L}{S\epsilon}\right), & \text{ if }S \le \sqrt{N} \\
	& O\left(\frac{N^{1/4}L}{\epsilon}\right), &\text{ if }\sqrt{N} < S \le N^{3/4} \\
	& O\left(\frac{NL}{S\epsilon}\right), & \text{ if }{N^{3/4}} < S
	\end{aligned}\right..\]
\end{restatable}

To understand this result, we can set $S = 1$, i.e., in each round as long as the server does not communicate with all clients, it only selects one client to communicate. Then, the total communication cost of \fedpage (here also the convergence result for \algname{PAGE}) becomes
$O\left(\frac{N^{3/4}}{\epsilon}\right).$
Recall that in the convex setting, \algname{SVRG}/\algname{SAGA} has convergence result $O\left(\frac{N}{\epsilon}\right)$. Thus \fedpage/\algname{PAGE} has much better convergence result compared with \algname{SVRG}/\algname{SAGA} in terms of the total number of clients $N$.

\subsection{\fedpageNORMAL with general local steps $K \ge 1$}\label{sec:convex-general}
In this part, we show the general result of \fedpage with any local steps $K \geq 1$ in the convex setting. Here we assume that $f$ is convex and the functions $\{f_{i,j}\}$ satisfy $L$-smoothness assumption (Assumption~\ref{ass:smooth}). The following Theorem~\ref{thm:convex} formally states the result.
\begin{restatable}[Convergence of \fedpage\ in convex setting]{theorem}{thmconvex}\label{thm:convex}
Under Assumption~\ref{ass:smooth} (and Assumption \ref{ass:bounded-variance}), if we choose the sampling probability $p_r = p = \frac{S}{N}$ for every $r \ge 1$ and $p_0 = 1$, the minibatch sizes $b_1 = \min\{M, \frac{24\sigma^2}{p^{1/2}\sqrt{S}\epsilon}\}, b_2 = \min\{M, \frac{48\sigma^2}{p^{3/2}\sqrt{S}\epsilon}\}$, and the global and local step size 
	\[\eta_g = \Theta\left(\frac{(S+N^{3/4})\frac{S}{N}}{L(S+\sqrt{N})}\right), \qquad \eta_l =  O\left(\frac{S}{N^{5/4}KL_c\sqrt{T}}\right),\]
	then \fedpage satisfies
	\[\frac{1}{R}\sum_{r=0}^{R-1}\E[f(x^{r+1})-f(x^*)] \le \left\{\begin{aligned}
	O\left(\frac{N^{3/4}L}{SR} + \epsilon\right), &\text{ if }S \le \sqrt{N} \\
	O\left(\frac{N^{1/4}L}{R} + \epsilon\right), &\text{ if }\sqrt{N} < S \le N^{3/4} \\
	O\left(\frac{NL}{SR} + \epsilon\right), &\text{ if } N^{3/4} < S
	\end{aligned}\right..\]
\end{restatable}

As we discussed before, the expected communication cost of \fedpage is the same as \algname{SCAFFOLD} in each communication round.
Then if the sampled clients $S\le \sqrt{N}$, \fedpage can find a solution $x$ such that $\E f(x) - f(x^*) \le \epsilon$ within $O(\frac{N^{3/4}L}{S\epsilon})$ number of communication rounds, improving the previous state-of-the-art $O(\frac{N}{S\epsilon})$ of \algname{SCAFFOLD} \citep{karimireddy2020scaffold} by a large factor of $N^{1/4}$. Recall that $N$ denotes the total number of clients.

\section{Numerical Experiments}\label{sec:exp}

In this section, we present our numerical experiments. We conducted two experiments: the first shows the effectiveness of the local steps (Section \ref{sec:exp-local}), and the second compares \fedpage with \algname{SCAFFOLD} and \algname{FedAvg} (Section \ref{sec:exp-comparison}). 
Before we present the results of these two experiments, we first state the experiment setups.
\paragraph{Experiment setup}
We run experiments on two nonconvex problems used in e.g. \citep{wang2018spiderboost,li2021zerosarah}: robust linear regression and logistic regression with nonconvex regularizer. 
The standard datasets \dataset{a9a} (32,561 samples) and \dataset{w8a} (49,749 samples) are downloaded from LIBSVM~\citep{chang2011libsvm}.  The objective function for robust linear regression is
\[\squeeze f(x) = \frac{1}{n}\sum \limits_{i=1}^n\ell(x^T a_i - b_i),\]
where $\ell(t) = \log(1+\frac{t^2}{2})$. Here $b_i\in\{\pm 1\}$ is a binary label.

The objective function for logistic regression with nonconvex regularizer is
\[\squeeze f(x) = \frac{1}{n}\sum \limits_{i=1}^n \log\left(1+\exp(-b_ix^T  a_i)\right) + \alpha\sum \limits_{j=1}^d\frac{x_j^2}{1+x_j^2}.\]
Here, the last term is the regularizer term and we set $\alpha=0.1$.

Besides, different algorithms have different definitions of the local step size and global step size, thus we compare these algorithms with the `effective step size' $\tilde\eta$. Here for \fedpage, the effective step size is just the global step size $\tilde\eta = \eta_g$, and for \algname{SCAFFOLD} and \algname{FedAvg}, the effective step size is defined as $\tilde\eta = K\eta_g\eta_l$. We run experiments with $\tilde\eta=0.1,0.03,0.01$. If the effective step size is larger, the algorithms may diverge. Also note that although we compare these algorithms with the same effective step size, \fedpage can use a larger step size from our theoretical results.
Finally we select the total number of communication rounds such that the algorithms converge or we can distinguish their performance difference.

\subsection{Effectiveness of local steps}\label{sec:exp-local}
In this experiment we compare the convergence performance of \fedpage using different number of local update steps: $K=1,10,20$ (see Line \ref{line:17} of Algorithm~\ref{alg:fedpage}). \algname{FedPAGE-1} means that the number of local step $K=1$, which reduces to the original \algname{PAGE} \citep{li2021page}, and \algname{FedPAGE-10} and \algname{FedPAGE-20} represent \fedpage with 10 and 20 local steps respectively.

We use the robust linear regression as the objective function. 
We run experiment on the \dataset{a9a} dataset in which the total number of data samples is $32500$ (here we drop the last 61 samples for easy implementation of different number of clients).
We choose the number of clients to be $3250, 325, 10$, and the numbers of data on a single client are $10, 100$, and $3250$, respectively. When the number of clients is $3250$, we choose $S=10$, i.e. the server communicates with $10$ clients in each communication round, and when the number of clients is $325$ or $10$, we set $S=1$. For all settings, we optimize the global step size $\eta_g$ and choose the local step size $\eta_l$ heuristically such that the algorithms converge as fast as possible. For \algname{FedPAGE-1} (or \algname{PAGE}), the local step size does not matter and choosing the optimal global step size achieves its best convergence rate, however for \algname{FedPAGE-10} and \algname{FedPAGE-20}, choosing $\eta_g,\eta_l$ with some heuristics does not guarantee the best performance. We also perform the similar experiments on another dataset \dataset{w8a}.

The experimental results are presented in Figure~\ref{fig:exp-local}. Figure~\ref{fig:exp-local-a9a} shows the robust linear regression results of \fedpage using different number of local steps $K=1, 10, 20$ on \dataset{a9a} dataset, and Figure~\ref{fig:exp-local-w8a} shows the result on \dataset{w8a} dataset.

\begin{figure}[!t]
	\centering
	\begin{subfigure}[b]{\textwidth}
		\centering
		\includegraphics[width=\textwidth]{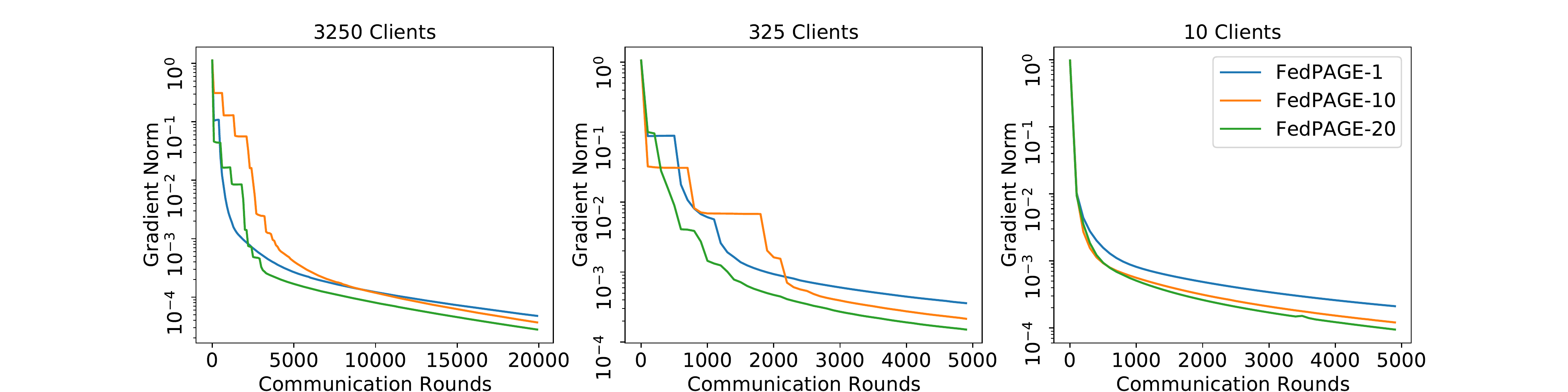}
		\caption{\dataset{a9a}.}
		\label{fig:exp-local-a9a}
	\end{subfigure}
	\begin{subfigure}[b]{\textwidth}
		\centering
		\includegraphics[width=\textwidth]{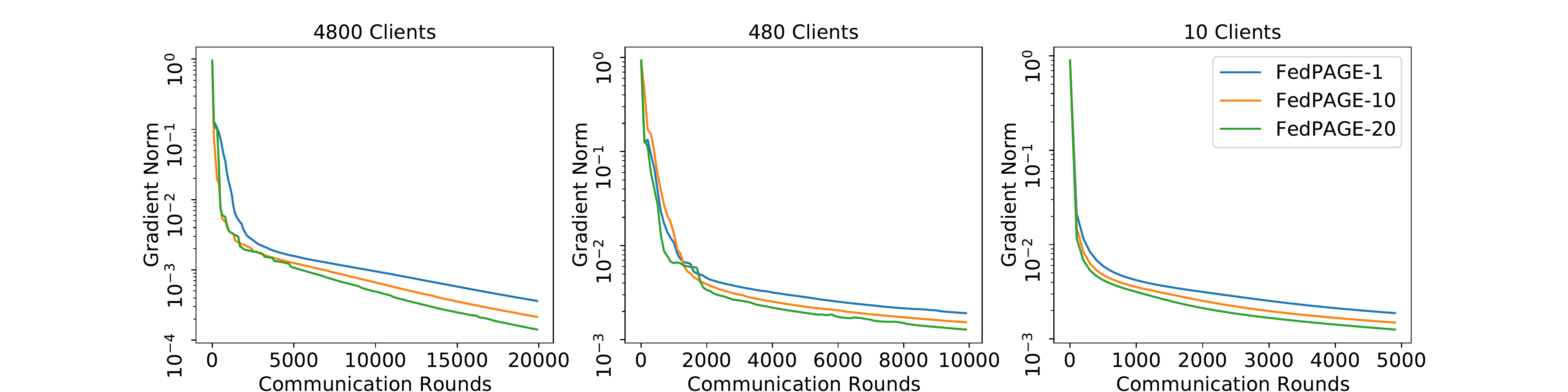}
		\caption{\dataset{w8a}.}
		\label{fig:exp-local-w8a}
	\end{subfigure}
	\caption{\fedpage with different number of local steps on different datasets.}
	\label{fig:exp-local}
\end{figure}

\paragraph{Local steps speed up the convergence rate} The experimental results in Figure~\ref{fig:exp-local} show that the multiple local steps of \fedpage can speed up the convergence in terms of the communication rounds. Although there are some fluctuations when the number of communication round is not large (early-stage), \algname{FedPAGE-10} and \algname{FedPAGE-20} outperform \algname{FedPAGE-1} in the end.

\paragraph{Algorithm with multiple local steps can choose a larger effective step size} From our hyperparameter optimization results, we also find that \fedpage with multiple local steps can choose a larger effective step size ($\eta_g$ in \fedpage). On \dataset{a9a} dataset, when there are $3250$ clients, the effective step size for \algname{FedPAGE-1}, \algname{FedPAGE-10}, and \algname{FedPAGE-20} are optimized to be $0.3, 0.4, 0.4$ respectively; when there are $325$ clients, the effective step size for \algname{FedPAGE-1}, \algname{FedPAGE-10}, and \algname{FedPAGE-20} are optimized to be $0.2, 0.4, 0.5$; when there are $10$ clients, the effective step size for \algname{FedPAGE-1}, \algname{FedPAGE-10}, and \algname{FedPAGE-20} are optimized to be $0.3, 0.5, 0.6$. The experiments on \dataset{w8a} dataset also support this finding.

\subsection{Comparison with previous methods}\label{sec:exp-comparison}
Now, we compare our \fedpage with two other methods: \algname{SCAFFOLD} \citep{karimireddy2020scaffold} and \algname{FedAvg} \citep{mcmahan2017communication}. The experimental results are presented in Figure~\ref{fig:exp-comparison-rlr} and \ref{fig:exp-comparison-lrnc}.
We plot the gradient norm versus the number of communication rounds. Figures~\ref{fig:robust-a9a}, \ref{fig:robust-w8a}, \ref{fig:lg-a9a}, and \ref{fig:lg-w8a} show the performance of each algorithm using different objective functions and different datasets.

For the experiments with \dataset{a9a} dataset, we omit the last 61 samples and set the number of clients to be 3250, and for experiments with \dataset{w8a}, we omit the last 1749 samples and there are 4800 clients in total. Here we omit the samples because it is more convenient to change the number of clients.
We let each `client' contains 10 samples from the datasets.
For \algname{SCAFFOLD} and \algname{FedAvg}, in each communication round, the server will communicate with 20 clients ($S=20$ in their algorithms). For \fedpage, we set $S = 10$ because \fedpage will communicate with all clients with probability $\frac{S}{N}$ and the expected communication for all three algorithms in each round are almost the same. We choose the local steps of all these three methods to be $10$. For \fedpage, we choose the minibatch size $b_3 = 1$ and for \algname{SCAFFOLD} and \algname{FedAvg}, we choose the minibatch size that estimate the local full gradient to be $4$. In this way, the local computations are nearly the same for all methods.

\begin{figure}[!b]
	\centering
	\begin{subfigure}[b]{\textwidth}
		\centering
		\includegraphics[width=\textwidth]{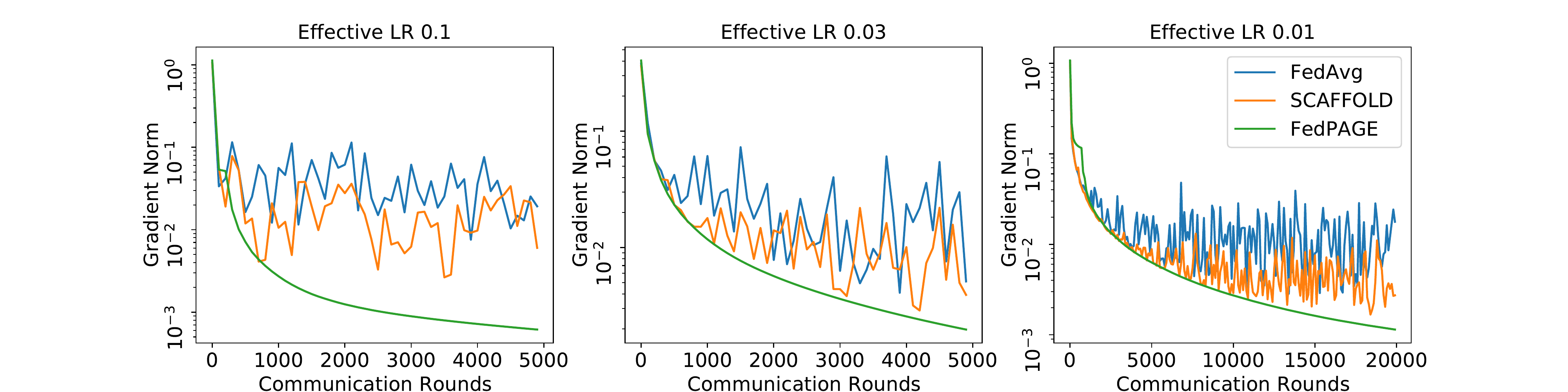}
		\caption{Robust linear regression on \dataset{a9a} dataset.}
		\label{fig:robust-a9a}
	\end{subfigure}
	\begin{subfigure}[b]{\textwidth}
		\centering
		\includegraphics[width=\textwidth]{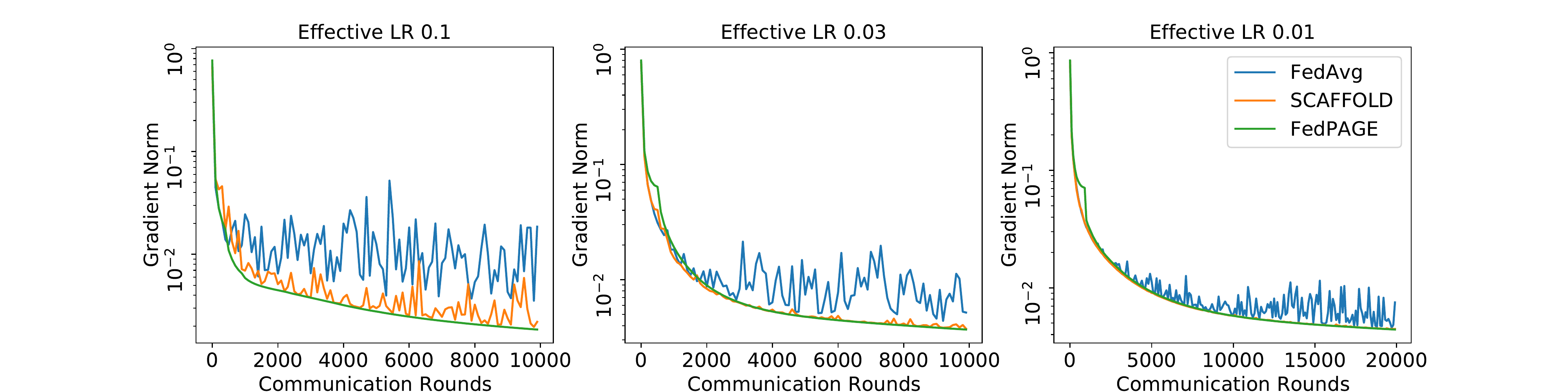}
		\caption{Robust linear regression on \dataset{w8a} dataset.}
		\label{fig:robust-w8a}
	\end{subfigure}
	\caption{Comparison of different methods with robust linear regression.}
	\label{fig:exp-comparison-rlr}
\end{figure}

\begin{figure}[!t]
	\centering
	\begin{subfigure}[b]{\textwidth}
		\centering
		\includegraphics[width=\textwidth]{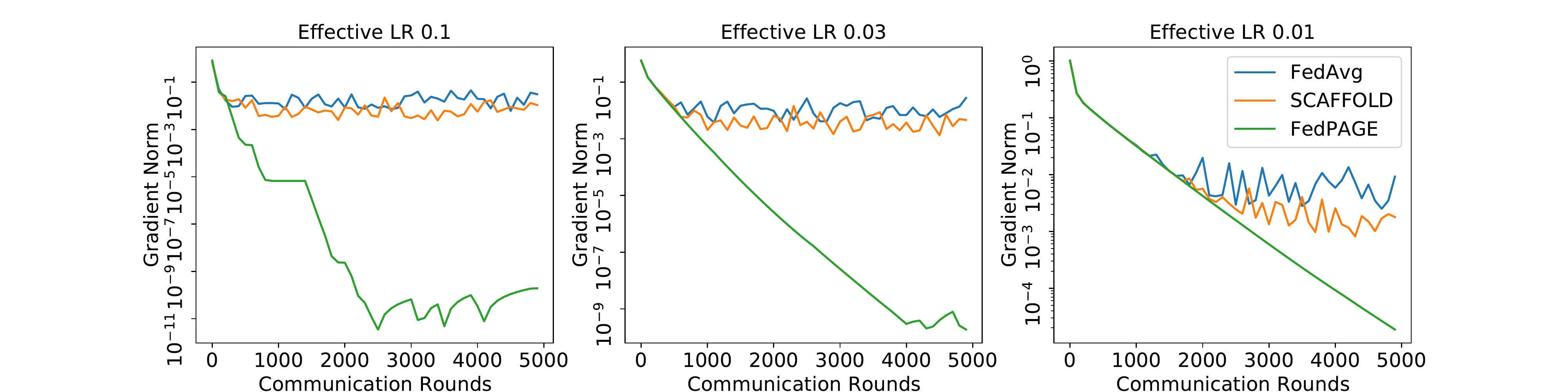}
		\caption{Logistic regression with nonconvex regularizer on \dataset{a9a} dataset.}
		\label{fig:lg-a9a}
	\end{subfigure}
	\begin{subfigure}[b]{\textwidth}
		\centering
		\includegraphics[width=\textwidth]{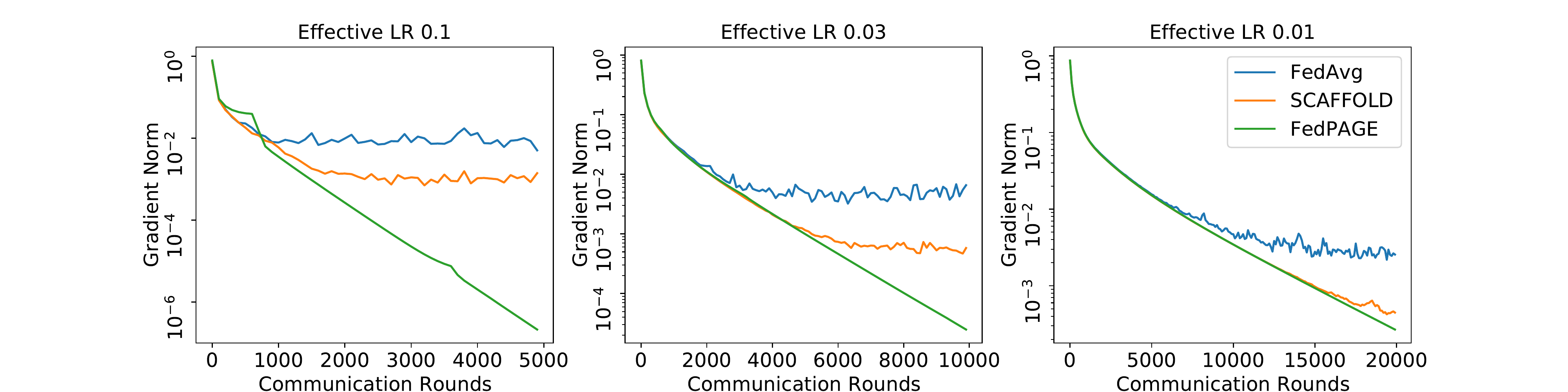}
		\caption{Logistic regression with nonconvex regularizer on \dataset{w8a} dataset.}
		\label{fig:lg-w8a}
	\end{subfigure}
	\caption{Comparison of different methods with logistic regression with nonconvex regularizer.}
	\label{fig:exp-comparison-lrnc}
\end{figure}

\paragraph{Performance of different methods} The experiments show that \fedpage $\ge$ \algname{SCAFFOLD} $>$ \algname{FedAvg}. Among all the cases, under the same effective step size, we find that both \fedpage and  \algname{SCAFFOLD} converge faster than \algname{FedAvg}. 
\fedpage converges at least as fast as \algname{SCAFFOLD}, and in most of the cases \fedpage converges faster than \algname{SCAFFOLD}. 

\paragraph{Larger effective step size converges faster} The experiments also show that a larger effective step size leads to a faster convergence as long as the algorithm converges. Note that \fedpage can use a larger step size with theoretical guarantee compared with \algname{SCAFFOLD}, if we choose the same parameters of the objective function (e.g. the same smoothness constant) and use the step size with theoretical guarantees, \fedpage converges faster than \algname{SCAFFOLD} than \algname{FedAvg}.

\section{Conclusion}
In this paper, we propose a new federated learning algorithm, \fedpage, providing much better state-of-the-art communication complexity for both federated convex and nonconvex optimization. 
Concretely, 
in the convex setting, the number of communication rounds of \fedpage is $O(\frac{N^{3/4}}{S\epsilon})$, which substantially improves previous best-known result $O(\frac{N}{S\epsilon})$ of \algname{SCAFFOLD} \citep{karimireddy2020scaffold} by a factor of $N^{1/4}$. 
In the nonconvex setting, the number of communication rounds of \fedpage is $O(\frac{\sqrt{N}+S}{S\epsilon^2})$, which also improves the best-known result  $O(\frac{N^{2/3}}{S^{2/3}\epsilon^2})$ of \algname{SCAFFOLD} \citep{karimireddy2020scaffold} by a large factor of $N^{1/6}S^{1/3}$.
Finally, we conduct several numerical experiments showing the effectiveness of multiple local update steps in \fedpage and verifying the practical superiority of \fedpage over other classical methods.

\clearpage
\bibliography{fedpage}
\bibliographystyle{plainnat}

\newpage
\input{appendix}

\end{document}

%% file: math_commands.tex
\usepackage{amsmath,amsfonts,bm}

\def\1{\bm{1}}

\makeatletter
\newcommand{\ve}{\@ifnextchar\bgroup{\velong}{{\bm{e}}}}
\newcommand{\velong}[1]{{\bm{#1}}}
\makeatother

\def\vu{{\bm{u}}}
\def\vv{{\bm{v}}}

\DeclareMathAlphabet{\mathsfit}{\encodingdefault}{\sfdefault}{m}{sl}
\SetMathAlphabet{\mathsfit}{bold}{\encodingdefault}{\sfdefault}{bx}{n}

\def\gA{{\mathcal{A}}}
\def\gB{{\mathcal{B}}}
\def\gC{{\mathcal{C}}}
\def\gD{{\mathcal{D}}}

\def\gF{{\mathcal{F}}}

\def\gI{{\mathcal{I}}}

\def\sI{{\mathbb{I}}}

\newcommand{\E}{\mathbb{E}}
\newcommand{\R}{\mathbb{R}}

%% file: symbols.tex
\def\eqref#1{(\ref{#1})}

%% file: appendix.tex
\appendix
\section*{Appendices}
\section{More Experiments}
In this section we present more numerical experiments. We perform two different experiments: the first is to compare the performance of different algorithm with different number of clients and data on a single client (Section \ref{sec:exp-clients}), and the second is to compare different algorithm with local full gradient computations, which shows the limitation of different algorithms (Section \ref{sec:exp-full}).

\subsection{Comparison of different methods with different number of clients}\label{sec:exp-clients}
\subsubsection{Experiment setup}
In previous Section \ref{sec:exp}, we compare different methods with a large number of clients (on \dataset{a9a} dataset, there are $3250$ clients, and on \dataset{w8a} dataset, there are $4800$ clients). In this experiment, we vary the number of clients and compare the performance of \fedpage, \algname{SCAFFOLD}, and \algname{FedAvg}.

For the number of clients, we choose the number of clients to be $325$ and $10$, and the number of data on a single client are $100$ and $3250$. We choose the number of local steps to be 10 for all three methods. When the number of clients are $325$ and $10$, we set $S=1$ for \fedpage and $S=2$ for \algname{SCAFFOLD} and \algname{FedAvg}, making the communication cost in each round to be nearly the same. We set \fedpage to compute the full local gradient for the first local step, and choose only one sample to estimate the gradient for the following local steps. For \algname{SCAFFOLD} and \algname{FedAvg}, we set the minibatch size estimating the local full gradient to be $22$ when the number of client is $325$ and $652$ when the number of client is $10$. When the number of client is $325$, there are $100$ data on a single client. \fedpage need to compute two full gradient at the beginning of each local computations, costing $200$ number of gradient computations. Then it needs to compute two gradient (the gradient of a same sample at different points), and it cost about $220$ gradient computations in total. Choosing the minibatch size to be $22$ in \algname{SCAFFOLD} and \algname{FedAvg} makes the local computations nearly the same, because \algname{SCAFFOLD} and \algname{FedAvg} use the same minibatch size in every local step. When the number of client is $10$, the minibatch size for \algname{SCAFFOLD} and \algname{FedAvg} can be computed as $3250\times 2 / 10 + 2 = 652$. This makes the local computations of these three algorithms to be nearly the same. 

For the step sizes, we choose the effective step sizes to be $0.1$, $0.03$, and $0.01$.

\subsubsection{Experiment results}

\begin{figure}[t]
    \centering
    \begin{subfigure}[b]{\textwidth}
        \centering
	    \includegraphics[width=\textwidth]{figs/grad_norm_c_3250_rlr_a9a.pdf}
	    \caption{3250 clients (10 data per client).}
	    \label{fig:exp-client-3250}
    \end{subfigure}
    
    \begin{subfigure}[b]{\textwidth}
        \centering
	    \includegraphics[width=\textwidth]{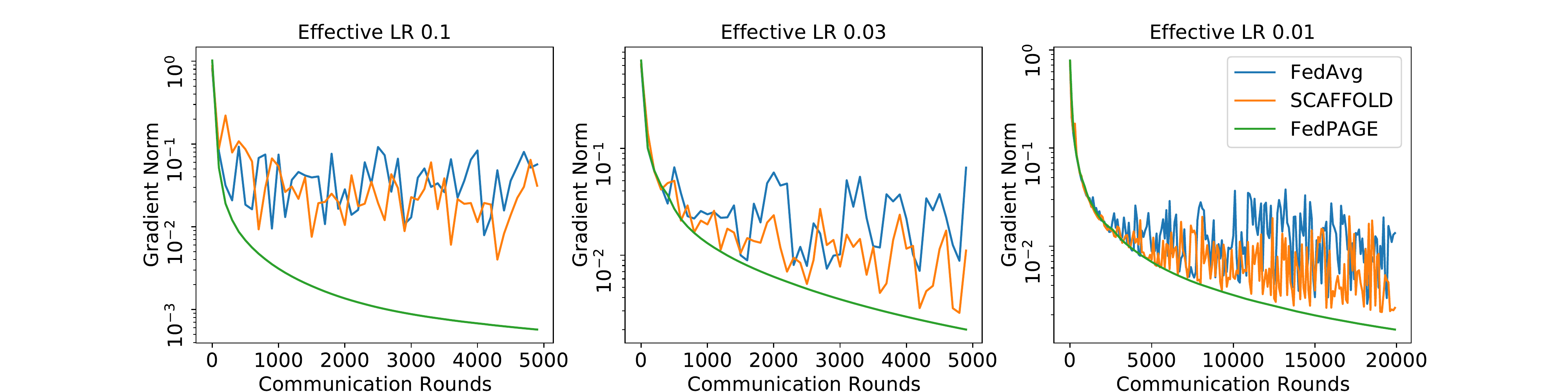}
	    \caption{325 clients (100 data per client).}
	    \label{fig:exp-client-325}
    \end{subfigure}
    
    \begin{subfigure}[b]{\textwidth}
        \centering
	    \includegraphics[width=\textwidth]{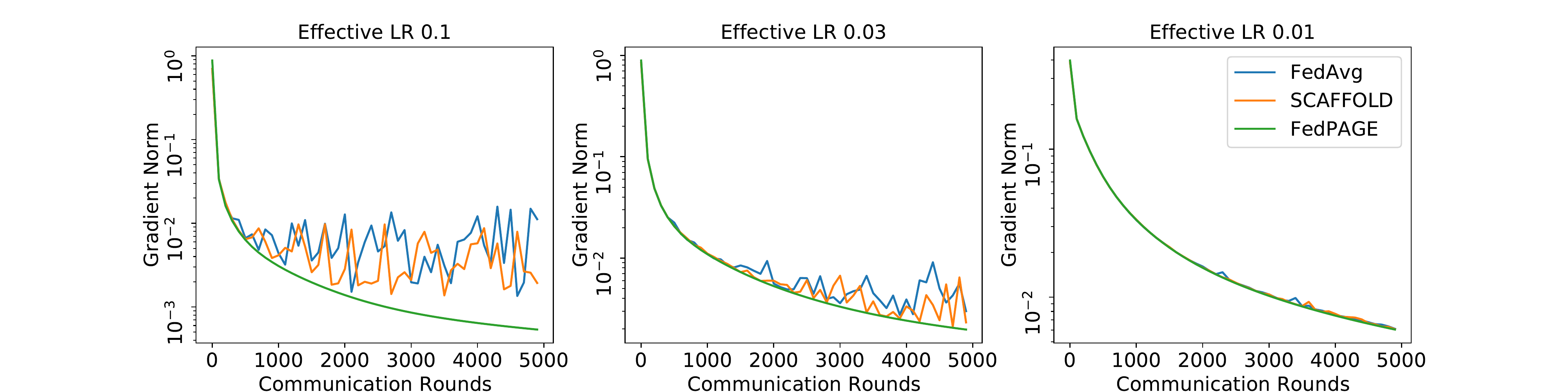}
	    \caption{10 clients (3250 data per client).}
	    \label{fig:exp-client-10}
    \end{subfigure}
    \caption{Experiment results of different methods with different number of clients.}
    \label{fig:exp-clients}
\end{figure}

The experimental results are presented in Figure~\ref{fig:exp-clients}. 
Figure~\ref{fig:exp-client-325} and \ref{fig:exp-client-10} shows the experiment results with 325 clients and 10 clients on \dataset{a9a} dataset. We also include Figure~\ref{fig:exp-client-3250} (i.e., Figure~\ref{fig:robust-a9a} in Section \ref{sec:exp-comparison}) with $3250$ clients for better comparison.
Similar to the experimental results in Section \ref{sec:exp},  Figure~\ref{fig:exp-clients} also demonstrates that \fedpage typically converges faster than \algname{SCAFFOLD} faster than \algname{FedAvg}. 

\newpage
\subsection{Comparison of different methods with local full gradient computation}\label{sec:exp-full}
\subsubsection{Experiment setup}
In this section, we design another experiment to observe the performance limitation of \fedpage, \algname{SCAFFOLD}, and \algname{FedAvg}. We substitute all the steps that use a minibatch to estimate the local full gradient to the actual full gradient computation. In \fedpage, we choose $b_3 = 1$ in the previous experiments and now we set $b_3 = M$, the number of data on a single client. We also choose $b_1 = b_2 = M$. We denote the resulting algorithm \algname{FedPAGE-Full}. 
Similarly, for \algname{SCAFFOLD} and \algname{FedAvg}, they choose a minibatch to estimate the local full gradient, and now we change them to computing the local full gradient, i.e., $b=M$. We denote the resulting algorithms as \algname{SCAFFOLD-Full} and \algname{FedAvg-Full}.

We then compare four different methods: \fedpage, \algname{FedPAGE-Full}, \algname{SCAFFOLD-Full}, and \algname{FedAvg-Full}. We perform the experiments on \dataset{a9a} and \dataset{w8a} datasets with robust linear regression objective and logistic regression with nonconvex regularizer objective. We let each `client' contains 10 samples from the dataset. We set all the algorithm to run with 10 local steps ($K=10$). We run the experiments with effective step size 0.1, 0.03, and 0.01. For experiment on \dataset{w8a} dataset with logistic regression with nonconvex regularizer, we also test the algorithms with effective step size $0.3$. For \fedpage and \algname{FedPAGE-Full}, we set $S=10$ and for \algname{SCAFFOLD-Full} and \algname{FedAvg-Full}, we set $S=20$ to make the communication cost in each round to be nearly the same.

\subsubsection{Experiment results}

\begin{figure}
    \centering
    \begin{subfigure}[b]{\textwidth}
        \centering
        \includegraphics[width=\textwidth]{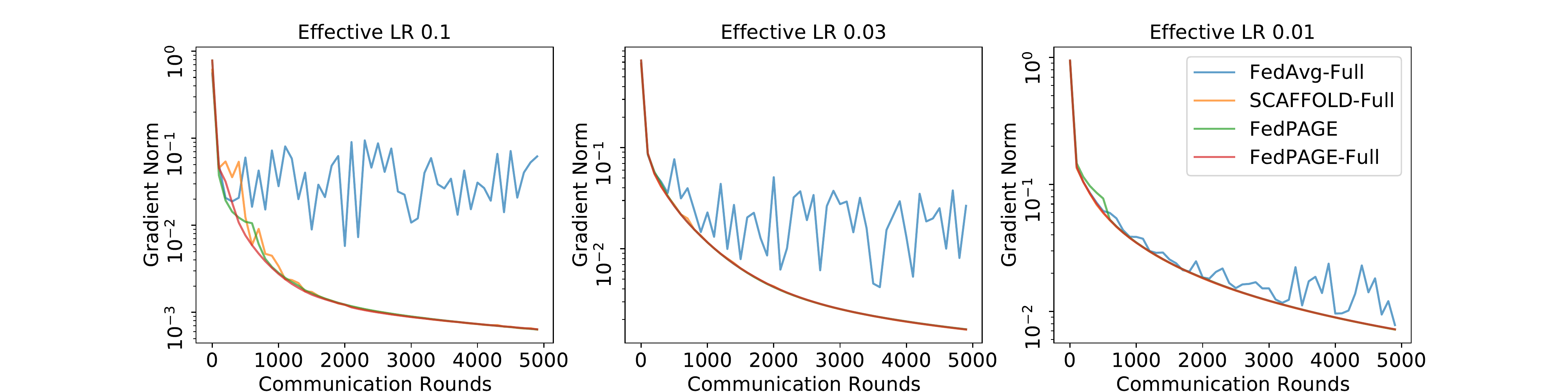}
        \caption{Robust linear regression on \dataset{a9a} dataset.}
        \label{fig:robust-a9a-full}
    \end{subfigure}
    \begin{subfigure}[b]{\textwidth}
        \centering
        \includegraphics[width=\textwidth]{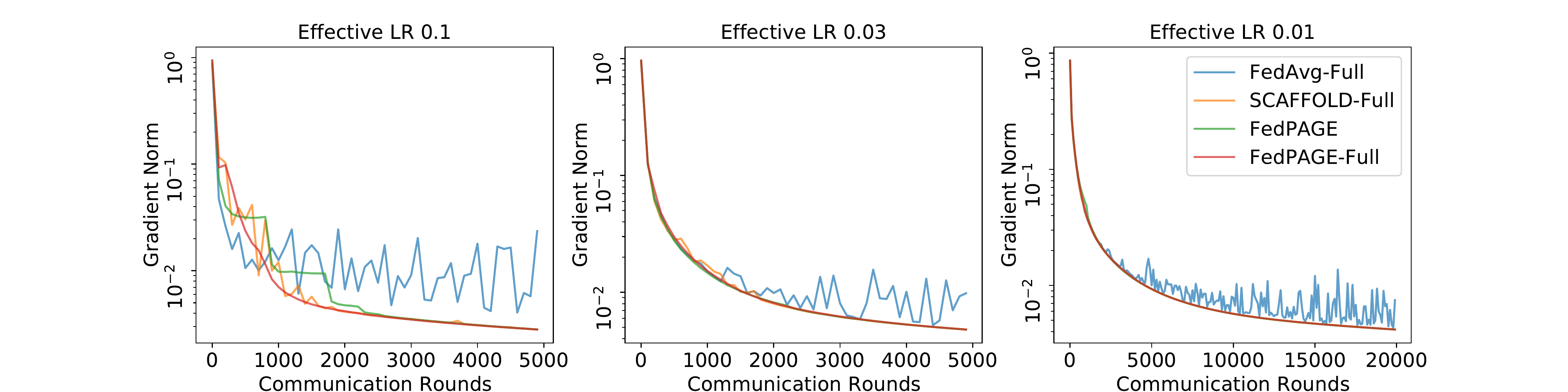}
        \caption{Robust linear regression on \dataset{w8a} dataset.}
        \label{fig:robust-w8a-full}
    \end{subfigure}
    \caption{Comparison of different methods with robust linear regression.}
    \label{fig:exp-comparison-rlr-full}
\end{figure}

\begin{figure}
    \centering
    \begin{subfigure}[b]{\textwidth}
        \centering
        \includegraphics[width=\textwidth]{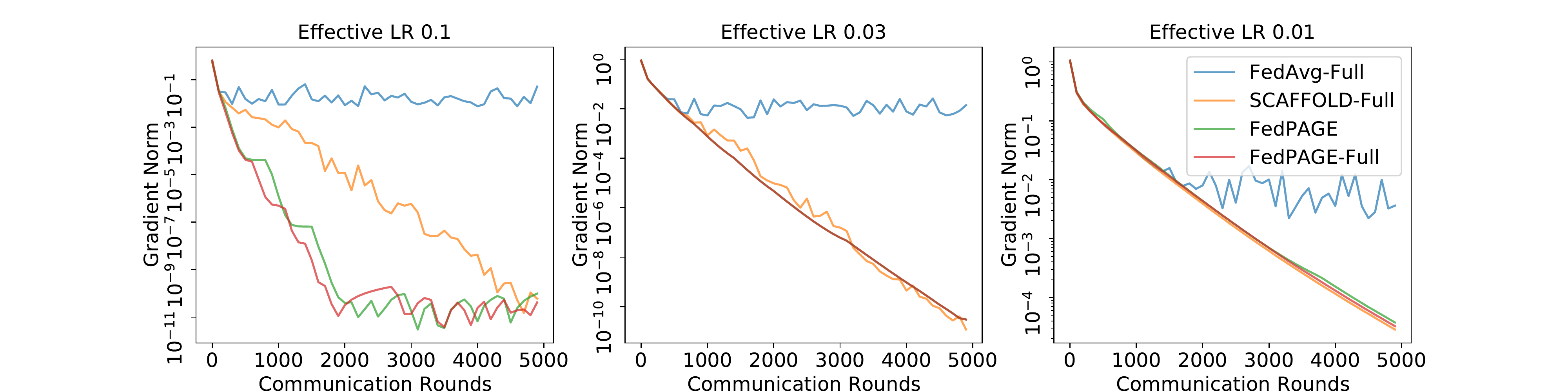}
        \caption{Logistic regression with nonconvex regularizer on \dataset{a9a} dataset.}
        \label{fig:lg-a9a-full}
    \end{subfigure}
    \begin{subfigure}[b]{\textwidth}
        \centering
        \includegraphics[width=\textwidth]{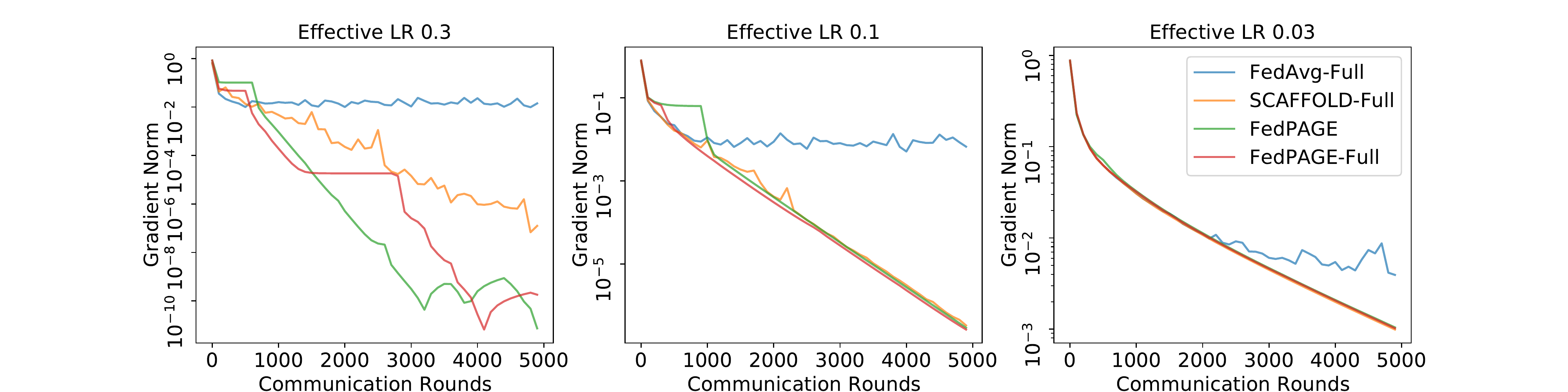}
        \caption{Logistic regression with nonconvex regularizer on \dataset{w8a} dataset.}
        \label{fig:lg-w8a-full}
    \end{subfigure}
    \caption{Comparison of different methods  with logistic regression with nonconvex regularizer.}
    \label{fig:exp-comparison-lrnc-full}
\end{figure}

The experimental results are presented in Figure~\ref{fig:exp-comparison-rlr-full} and \ref{fig:exp-comparison-lrnc-full}. Figure~\ref{fig:robust-a9a-full}, \ref{fig:robust-w8a-full}, \ref{fig:lg-a9a-full}, and \ref{fig:lg-w8a-full} show the experimental results of different methods on different problems and different datasets as stated in their captions.

\paragraph{\fedpage $\approx$ \algname{FedPAGE-Full}} First, the experiments show that the convergence performance of \fedpage and \algname{FedPAGE-Full} are nearly the same under the same effective step size. Although there are some fluctuations in the convergence process, the fluctuations are not large enough to conclude any difference between the convergence speed of \fedpage and \algname{FedPAGE-Full}.

\paragraph{\algname{FedPAGE-Full} $\ge$ \algname{SCAFFOLD-Full} $>$ \algname{FedAvg-Full}} Next, the experiments show that \algname{FedPAGE-Full} converges at least as fast as (usually outperforms) \algname{SCAFFOLD-Full} and both of them converge faster than \algname{FedAvg-Full} in all cases. Using the robust linear regression objective in Figure~\ref{fig:exp-comparison-rlr-full}, \algname{FedPAGE-Full} and \algname{SCAFFOLD-Full} converges nearly at the same speed, but in the experiments with logistic regression with nonconvex regularizer in Figure~\ref{fig:exp-comparison-lrnc-full}, \algname{FedPAGE-Full} usually outperforms \algname{SCAFFOLD-Full} especially when the effective step size is large. 
From the experiments, \fedpage either has faster convergence performance under the same local computation cost, or can use less local computational resources and achieve the same or even better performance.

\section{Gradient Complexity of Different Methods}

\begin{table}[!t]
	\caption{Number of gradient computations per client for finding an $\epsilon$-solution of federated convex and nonconvex problems \eqref{eq:problem-setting}.}
	\label{tab:compuataion-results}
	\centering
	\small
	\begin{tabular}{|c|c|c|c|c|}
		\hline
		Algorithm & Convex setting & Nonconvex setting & Assumption\\
		\hline
		\makecell{\algname{FedAvg} \\ \citep{yu2019parallel}} & --- & $\frac{G^2NK^2}{\epsilon^2} + \frac{\sigma^2}{N\epsilon^4}$ & \makecell{Smooth, BV \\ $(G,0)$-BGD} \\
		\hline
		\makecell{\algname{FedAvg} \\ \citep{karimireddy2020scaffold}} & $\frac{G^2K}{N\epsilon^2} + \frac{GSK}{N\epsilon^{3/2}} + \frac{B^2SK}{N\epsilon} + \frac{\sigma^2}{N\epsilon^2}$ & $\frac{G^2K}{N\epsilon^4} + \frac{GSK}{N\epsilon^3} + \frac{B^2SK}{N\epsilon^2} + \frac{\sigma^2}{N\epsilon^4}$  & \makecell{Smooth, BV \\ $(G,B)$-BGD} \\
		\hline
		\makecell{\algname{FedProx} \\ \citep{sahu2018convergence}} & $\frac{B^2}{\epsilon}$ & --- & \makecell{Smooth, $S = N$, \\ $(0,B)$-BGD} \\
		\hline
		\makecell{\algname{VRL-SGD} \\ \citep{liang2019variance}} & --- & $\frac{N K}{\epsilon^2} + \frac{N \sigma^2}{\epsilon^4}$ & \makecell{Smooth, BV,\\ $S = N$} \\
		\hline
		\makecell{\algname{S-Local-SVRG}\\ \citep{gorbunov2020local}} & \makecell{$\frac{K+\sqrt{M/N}+M^{1/3}K^{2/3}}{\epsilon}$} & --- & \makecell{Smooth, (BV),\\ $S = N,~ K\leq M$}\\
		\hline
		\makecell{\algname{SCAFFOLD}\\ \citep{karimireddy2020scaffold}} & $\frac{K}{\epsilon} + \frac{\sigma^2}{N\epsilon^2}$ & $\frac{S^{1/3}K}{N^{1/3}\epsilon^2} + \frac{\sigma^2}{N\epsilon^4}$ & Smooth, BV \\
		\hline
		\rowcolor{bgcolor}
		\gape{\makecell{\fedpage \\ (this paper) }}& \gape{\makecell{$\frac{N^{3/4}}{N\epsilon}(M+K)$}} & $\frac{N^{1/2}+S}{N\epsilon^2}(M+K)$ & Smooth \\
		\hline
		\rowcolor{bgcolor}
		\gape{\makecell{\fedpage \\ (this paper) }}& \gape{\makecell{$\frac{N^{3/4}}{N\epsilon}\left(\frac{N^{3/2}\sigma^2}{S\epsilon^2}+K\right)$}}   & $\frac{N^{1/2}+S}{N\epsilon^2}\left(\frac{N\sigma^2}{S\epsilon^2}+K\right)$ & Smooth, BV \\
		\hline
	\end{tabular}
\end{table} 

In previous Table \ref{tab:results}, we show the number of communication rounds of different methods. In this section, we compare the gradient complexity among different methods. Table \ref{tab:compuataion-results} summarizes the gradient complexity per client of different methods under different assumptions.

For \algname{SCAFFOLD}, in each communication round, $S$ selected clients need to perform $K$ local steps, and the gradient computations of local client is the number of communication round times $SK/N$. For \fedpage, in each communication round, $S$ selected clients need to first compute two full/moderate minibatch gradients, and then performs $K$ local steps computing only $O(1)$ number of gradient in each step. The gradient complexity per client of \fedpage is the number of communication round times $S(M+K)/N$. In the BV setting, the full gradient may not be available, then \fedpage uses a moderate minibatch gradient to estimate the full gradient, and one only needs to change $M$ to the moderate minibatch size in order to obtain the corresponding gradient complexity (See the last two rows in Table \ref{tab:compuataion-results}).

In particular, if the number of data on a single client/device is not very large ($M$ is not very large), one can choose $K$ such that $M+K = O(K)$. Then the number of gradient computed by \fedpage during a communication round is similar to that computed by \algname{SCAFFOLD}, and also the number of communication rounds of \fedpage is much smaller than that of \algname{SCAFFOLD} regardless of settings (see Table \ref{tab:results}). As a result, \fedpage is strictly much better than \algname{SCAFFOLD} in terms of both communication complexity and computation complexity, both by a factor of $N^{1/4}$ in the convex setting and $N^{1/6}S^{1/3}$ in the nonconvex setting.
Thus, \fedpage is more suitable for the federated learning tasks that have many devices and each device has limited number of data, e.g. mobile phones.

\section{Technical Lemmas}
In this part we recall some classical inequalities that helps our derivation.
\begin{proposition}
    Let $\{\vv_1,\dots,\vv_{\tau}\}$ be $\tau$ vectors in $\R^d$. Then,
    \begin{align}
        & \langle \vv_i,\vv_j\rangle \le \frac{c}{2}\|\vu\|^2 + \frac{1}{2c}\|\vv\|^2, \forall c > 0.\label{eq:cauchy}\\
        &\|\vv_i + \vv_j\|^2 \le (1 + \alpha)\|\vv_i\|^2 + \left(1+\frac{1}{\alpha}\right)\|\vv_j\|^2, \forall \alpha > 0. \label{eq:rti-1}\\
        &\left\|\sum_{i=1}^{\tau}\vv_i\right\|^2\le\tau\sum_{i=1}^{\tau}\|\vv_i\|^2.\label{eq:rti-2}\\
    \end{align}
    
\end{proposition}

\begin{proposition}
If $X\in \R^d$ is a random variable, then
    \begin{equation}
        \E\|X\|^2 = \E\|X - \E X\|^2 + \|\E X\|^2.\label{eq:mean-var-separation}
    \end{equation}
    Besides, we have
    \begin{equation}
        \E\|X - \E X\|^2 \le \E\|X\|^2.\label{eq:mean-var-ineq}
    \end{equation}
    If $X,Y\in \R^d$ are independent random variables and $\E Y = \mathbf 0$, then we have
    \begin{equation}
        \E\|X + Y\|^2 = \E\|X\|^2 + \E\|Y\|^2.\label{eq:sep-variance}
    \end{equation}
    If $X_1,\dots,X_n\in\R^d$ are independent random variables and $\E X_i = \mathbf 0$ for all $i$, then
    \begin{equation}
        \E\left\|\sum_{i=1}^n X_i\right\|^2 = \sum_{i=1}^n \E\|X_i\|^2.\label{eq:indep-variance}
    \end{equation}
\end{proposition}

\begin{proposition}
    If $X,Y\in\R^d$ are two random variables (possibly dependent), then
    \begin{equation}
        \E\|X + Y\|^2 \le \|\E X + \E Y\|^2 + 2\E\|X - \E X\|^2 + 2\E\|Y- \E Y\|^2.\label{eq:mean-var-separation-2-variables}
    \end{equation}
\end{proposition}
\begin{proof}
    \begin{align*}
        \E\|X+Y\|^2 \overset{(\ref{eq:mean-var-separation})}{=}& \|\E X + \E Y\|^2 + \E\|X+Y - \E(X+Y)\|^2 \\
        \overset{(\ref{eq:rti-1})}{\le}& \|\E X + \E Y\|^2 + 2\E\|X - \E X\|^2 + 2\E\|Y - \E Y\|^2.
    \end{align*}
\end{proof}

%\newpage
\section{Missing Proofs in Section \ref{sec:nonconvex}}
In this section, we prove the convergence result of \fedpage in the nonconvex setting (Theorem~\ref{thm:non-convex}).

We use $\E_r$ to denote the expectation after $x^r$ is determined. Recall that we assume that $\{f_{i,j}\}$ are $L$-smooth, and formally, we have the following assumption

\asssmooth*

\begin{lemma}\label{lem:local-err}
	Under Assumption~\ref{ass:smooth}, if we choose $b_3 = 1$ and the local step size $\eta_l \le \frac{\sqrt{2}p}{24\sqrt{S}KL}$ in \fedpage , we have for any $i,k,r$
	\begin{align*}
	& \frac{1}{K}\sum_{k=0}^{K-1}\E_r\|g_{i,k}^r - g_{i,0}^r\|^2 \\
	\le& 12K^2 L^2\eta_l^2\left(\frac{\sigma^2\sI\{b_2 < M\}}{b_2} + L^2\|x^r - x^{r-1}\|^2 + \|g^{r-1} - \nabla f(x^{r-1})\|^2 + \|\nabla f(x^{r-1})\|^2\right)
	\end{align*}
\end{lemma}
\begin{proof}
    For any $i,k,r$, we have
	\begin{align}
	&\E_r\|g_{i,k}^r - g_{i,0}^r\|^2 \nonumber\\
	=& \E_r\|\nabla_{\gI_3} f_{i}(y_{i,k}^r) - \nabla_{\gI_3} f_{i}(y_{i,k-1}^r) + g_{i,k-1}^r  - g_{i,0}^r\|^2 \nonumber\\
	\overset{(\ref{eq:rti-1})}{=}& \left(1 + \frac{1}{K-1}\right)\|g_{i,k-1}^r  - g_{i,0}^r\|^2 + K\E_r\|\nabla_{\gI_3} f_{i}(y_{i,k}^r) - \nabla_{\gI_3} f_{i}(y_{i,k-1}^r)\|^2 \nonumber\\
	\le& \left(1 + \frac{1}{K-1}\right)\E_r\|g_{i,k-1}^r  - g_{i,0}^r\|^2 + K L^2\E_r\|y_{i,k}^r - y_{i,k-1}^r\|^2 \label{eq:b3}\\
	=& \left(1 + \frac{1}{K-1}\right)\E_r\|g_{i,k-1}^r  - g_{i,0}^r\|^2 + K L^2\eta_l^2 \E_r\|g_{i,k-1}^r\|^2 \nonumber \\
	=& \left(1 + \frac{1}{K-1}\right)\E_r\|g_{i,k-1}^r  - g_{i,0}^r\|^2 + K L^2\eta_l^2 \E_r\|g_{i,k-1}^r - g_{i,0}^r + g_{i,0}^r\|^2 \nonumber \\
	=& \left(1 + \frac{1}{K-1}\right)\E_r\|g_{i,k-1}^r  - g_{i,0}^r\|^2 + K L^2\eta_l^2 \E_r\|g_{i,k-1}^r - g_{i,0}^r + \nabla_{\gI_2} f_i(x^r) - \nabla_{\gI_2} f_i(x^{r-1}) + g^{r-1}\|^2 \nonumber \\
	\overset{(\ref{eq:mean-var-separation-2-variables})}{\le}& \left(1 + \frac{1}{K-1}\right)\E_r\|g_{i,k-1}^r  - g_{i,0}^r\|^2 + K L^2\eta_l^2 \E_r\|g_{i,k-1}^r - g_{i,0}^r + \nabla f_i(x^r) - \nabla f_i(x^{r-1}) + g^{r-1}\|^2\nonumber \\
	&\quad +4KL^2\eta_l^2\frac{\sigma^2\sI\{b_2 < M\}}{b_2}\nonumber \\
	=& \left(1 + \frac{1}{K-1}\right)\E_r\|g_{i,k-1}^r  - g_{i,0}^r\|^2+4KL^2\eta_l^2\frac{\sigma^2\sI\{b_2 < M\}}{b_2}\nonumber \\
	&\quad + K L^2\eta_l^2 \E_r\|g_{i,k-1}^r - g_{i,0}^r + \nabla f_i(x^r) - \nabla f_i(x^{r-1}) + g^{r-1} - \nabla f(x^{r-1}) + \nabla f(x^{r-1})\|^2 \nonumber \\
	\overset{(\ref{eq:rti-2})}{\le}& \left(1 + \frac{1}{K-1}\right)\E_r\|g_{i,k-1}^r  - g_{i,0}^r\|^2 +4KL^2\eta_l^2\frac{\sigma^2\sI\{b_2 < M\}}{b_2} + 4K L^2\eta_l^2 \E_r\|g_{i,k-1}^r - g_{i,0}^r\|^2\nonumber \\
	&\quad + 4K L^2\eta_l^2\|\nabla f_i(x^r) - \nabla f_i(x^{r-1})\|^2 + 4K L^2\eta_l^2\|g^{r-1} - \nabla f(x^{r-1})\|^2 + 4K L^2\eta_l^2\|\nabla f(x^{r-1})\|^2 \nonumber \\
	=& \left(1 + \frac{1}{K-1} + 4K L^2\eta_l^2\right)\E_r\|g_{i,k-1}^r  - g_{i,0}^r\|^2 +4KL^2\eta_l^2\frac{\sigma^2\sI\{b_2 < M\}}{b_2}\nonumber \\
	&\quad + 4K L^2\eta_l^2\|\nabla f_i(x^r) - \nabla f_i(x^{r-1})\|^2 + 4K L^2\eta_l^2\|g^{r-1} - \nabla f(x^{r-1})\|^2 + 4K L^2\eta_l^2\|\nabla f(x^{r-1})\|^2 \nonumber \\
	\le& 4 K L^2 \eta_l^2\left(\frac{\sigma^2\sI\{b_2 < M\}}{b_2} + \|\nabla f_i(x^r) - \nabla f_i(x^{r-1})\|^2 + \|g^{r-1} - \nabla f(x^{r-1})\|^2 + \|\nabla f(x^{r-1})\|^2\right)\nonumber \\
	&\quad \cdot\left(\sum_{k'=0}^{k-1}(1+\frac{1}{K-1} + 4KL^2\eta_l^2)^{k'}\right)\nonumber \\
	\le& 4K L^2 \eta_l^2\left(\frac{\sigma^2\sI\{b_2 < M\}}{b_2} + \|\nabla f_i(x^r) - \nabla f_i(x^{r-1})\|^2 + \|g^{r-1} - \nabla f(x^{r-1})\|^2 + \|\nabla f(x^{r-1})\|^2\right)\nonumber \\
	&\quad \cdot\left(\sum_{k=0}^{K-1}(1+\frac{1}{K-1} + 4KL^2\eta_l^2)^k\right)\nonumber \\
	\le& 12K^2 L^2\eta_l^2\left(\frac{\sigma^2\sI\{b_2 < M\}}{b_2} + \|\nabla f_i(x^r) - \nabla f_i(x^{r-1})\|^2 + \|g^{r-1} - \nabla f(x^{r-1})\|^2 + \|\nabla f(x^{r-1})\|^2\right) \label{eq:compute-K} \\
	\le& 12K^2 L^2\eta_l^2\left(\frac{\sigma^2\sI\{b_2 < M\}}{b_2} + L^2\|x^r - x^{r-1}\|^2 + \|g^{r-1} - \nabla f(x^{r-1})\|^2 + \|\nabla f(x^{r-1})\|^2\right). \nonumber
	\end{align}
	In the derivation, (\ref{eq:b3}) comes from the smoothness assumption (Assumption~\ref{ass:smooth}), (\ref{eq:compute-K}) comes from the fact that if we choose $\eta_l \le \frac{\sqrt{2}p}{24\sqrt{S}KL}$, then we have
	\begin{align*}
	&\sum_{k=0}^{K-1}\left(1+\frac{1}{K-1} + 4KL^2\eta_l^2\right)^k\\
	\le& \frac{\left(1+\frac{1}{K-1} + 4KL^2\eta_l^2\right)^K-1}{\frac{1}{K-1} + 4KL^2\eta_l^2}\\
	\le& (K-1)\left(\left(1+\frac{1}{K-1} + 4KL^2\eta_l^2\right)^K-1\right)\\
	\le& (K-1)\left(\left(1+\frac{1}{K-1} + \frac{1}{36K}\right)^K-1\right)\\
	\le& 3K,
	\end{align*}
	for any $K\ge 2$. Then we take the average over $k$, we get
	\begin{align*}
	&\frac{1}{K}\sum_{k=0}^{K-1}\E_r\|g_{i,k}^r - g_{i,0}^r\|^2 \\
	\le& 12K^2 L^2\eta_l^2\left(\frac{\sigma^2\sI\{b_2 < M\}}{b_2} + L^2\|x^r - x^{r-1}\|^2 + \|g^{r-1} - \nabla f(x^{r-1})\|^2 + \|\nabla f(x^{r-1})\|^2\right),
	\end{align*}
	and we conclude the proof of this lemma.
\end{proof}

\begin{lemma}\label{lem:gradient-error}
	Under Assumption~\ref{ass:smooth}, if we choose $b_3 = 1$ in \fedpage, the local step size $\eta_l \le \frac{\sqrt{2}p}{24\sqrt{S}KL}$, the batch sizes $b_1 = \min\{M, \frac{24\sigma^2}{pN\epsilon^2}\}, b_2 = \min\{M, \frac{24\sigma^2}{pS\epsilon^2}\}$, then we have
	\begin{equation*}
	\E_r||g^r - \nabla f(x^r)||^2
	\le (1-\frac{p}{3})\| g^{r-1} - \nabla f(x^{r-1})\|^2 + \frac{1-p/3}{S}L^2\E_r\|x^r - x^{r-1}\|^2 + \frac{p}{6S}\|\nabla f(x^{r-1})\|^2 + \frac{p\epsilon^2}{8}.
	\end{equation*}
\end{lemma}

\begin{proof}
	\begin{align}
	&\E_r\left\|g^r - \nabla f(x^r)\right\|^2 \nonumber\\
	=& (1-p)\E_r\left\|\frac{1}{K|S^r|}\sum_{i\in S^r}\sum_{k=1}^K g_{i,k-1}^r - \nabla f(x^r)\right\|^2 + p\left\|\frac{1}{N}\sum_{i\in [N]} \nabla_{\gI_1} f_i(x^r) - \nabla f(x^r)\right\|^2\nonumber\\
	=& (1-p)\E_r\left\|\frac{1}{K|S^r|}\sum_{i\in S^r}\sum_{k=1}^K g_{i,k-1}^r - \nabla f(x^r)\right\|^2 + \frac{p\sigma^2\sI\{b_1 < M\}}{N b_1}\nonumber\\
	=& (1-p)\E_r\left\|\frac{1}{K|S^r|}\sum_{i\in S^r}\sum_{k=1}^K \left(g_{i,k-1}^r - g_{i,0}^r + g_{i,0}^r - \nabla f(x^r)\right)\right\|^2+ \frac{p\sigma^2\sI\{b_1 < M\}}{N b_1}\nonumber\\
	=& (1-p)\E_r\left\|\frac{1}{K|S^r|}\sum_{i\in S^r}\sum_{k=1}^K \left(g_{i,k-1}^r - g_{i,0}^r + \nabla_{\gI_2} f_i(x^r) - \nabla_{\gI_2} f_i(x^{r-1}) + g^{r-1} - \nabla f(x^r)\right)\right\|^2\nonumber\\
	&\quad + \frac{p\sigma^2\sI\{b_1 < M\}}{N b_1} \nonumber\\
	\overset{(\ref{eq:mean-var-separation})}{=}& (1-p)\left\|\E_r\frac{1}{K|S^r|}\sum_{i\in S^r}\sum_{k=1}^K \left(g_{i,k-1}^r - g_{i,0}^r  + g^{r-1} - \nabla f(x^{r-1})\right)\right\|^2+ \frac{p\sigma^2\sI\{b_1 < M\}}{N b_1}\label{eq:seperate-mean-var-1}\\
	&\quad + (1-p)\E_r\bigg\|\frac{1}{K|S^r|}\sum_{i\in S^r}\sum_{k=1}^K \left(g_{i,k-1}^r - g_{i,0}^r\right) - \E_r\frac{1}{K|S^r|}\sum_{i\in S^r}\sum_{k=1}^K \left(g_{i,k-1}^r - g_{i,0}^r\right)\nonumber\\
	&\quad\quad + \frac{1}{K|S^r|}\sum_{i\in S^r}\sum_{k=1}^K\left(\nabla_{\gI_2} f_i(x^r) - \nabla f(x^r) + \nabla f(x^{r-1}) - \nabla_{\gI_2} f_i(x^{r-1})\right)\bigg\|^2 \nonumber\\
	\overset{(\ref{eq:rti-1})}{\le}& (1-p)\left(1+\frac{p}{2}\right)\| g^{r-1} - \nabla f(x^{r-1})\|^2 + \frac{p\sigma^2\sI\{b_1 < M\}}{N b_1}\nonumber\\
	&\quad + (1-p)\left(1 + \frac{2}{p}\right)\left\|\E_r\frac{1}{K|S^r|}\sum_{i\in S^r}\sum_{k=1}^K \left(g_{i,k-1}^r - g_{i,0}^r\right)\right\|^2 \nonumber\\
	&\quad + (1-p)\left(1+ \frac{p}{2}\right)\E_r\left\|\frac{1}{K|S^r|}\sum_{i\in S^r}\sum_{k=1}^K\left(\nabla_{\gI_2} f_i(x^r) - \nabla f(x^r) + \nabla f(x^{r-1}) - \nabla_{\gI_2} f_i(x^{r-1})\right)\right\|^2 \nonumber\\
	&\quad + (1-p)\left(1 + \frac{2}{p}\right) \E_r\left\|\frac{1}{K|S^r|}\sum_{i\in S^r}\sum_{k=1}^K \left(g_{i,k-1}^r - g_{i,0}^r\right) - \E_r\frac{1}{K|S^r|}\sum_{i\in S^r}\sum_{k=1}^K \left(g_{i,k-1}^r - g_{i,0}^r\right)\right\|^2\nonumber\\
	=& (1-p)\left(1+\frac{p}{2}\right)\| g^{r-1} - \nabla f(x^{r-1})\|^2 + \frac{p\sigma^2\sI\{b_1 < M\}}{N b_1}\nonumber\\
	&\quad + (1-p)\left(1 + \frac{2}{p}\right)\left\|\E_r\frac{1}{K|S^r|}\sum_{i\in S^r}\sum_{k=1}^K \left(g_{i,k-1}^r - g_{i,0}^r\right)\right\|^2 \nonumber\\
	&\quad + (1-p)\left(1+ \frac{p}{2}\right)\E_r\left\|\frac{1}{|S^r|}\sum_{i\in S^r}\left(\nabla_{\gI_2} f_i(x^r) - \nabla f(x^r) + \nabla f(x^{r-1}) - \nabla_{\gI_2} f_i(x^{r-1})\right)\right\|^2 \nonumber\\
	&\quad + (1-p)\left(1 + \frac{2}{p}\right) \E_r\left\|\frac{1}{K|S^r|}\sum_{i\in S^r}\sum_{k=1}^K \left(g_{i,k-1}^r - g_{i,0}^r\right) - \E_r\frac{1}{K|S^r|}\sum_{i\in S^r}\sum_{k=1}^K \left(g_{i,k-1}^r - g_{i,0}^r\right)\right\|^2\nonumber\\
	\overset{(\ref{eq:indep-variance})}{\le}& (1-p)\left(1+\frac{p}{2}\right)\| g^{r-1} - \nabla f(x^{r-1})\|^2 + \frac{p\sigma^2\sI\{b_1 < M\}}{N b_1}\nonumber\\
	&\quad + (1-p)\left(1 + \frac{2}{p}\right)\left\|\E_r\frac{1}{K|S^r|}\sum_{i\in S^r}\sum_{k=1}^K \left(g_{i,k-1}^r - g_{i,0}^r\right)\right\|^2 \nonumber\\
	&\quad + \frac{1-p}{S}\left(1+ \frac{p}{2}\right)\E_r\left\|\nabla_{\gI_2} f_i(x^r) - \nabla f(x^r) + \nabla f(x^{r-1}) - \nabla_{\gI_2} f_i(x^{r-1})\right\|^2 \label{eq:variance}\\
	&\quad \frac{1-p}{S}\left(1 + \frac{2}{p}\right) \E_r\left\|\frac{1}{K}\sum_{k=1}^K \left(g_{i,k-1}^r - g_{i,0}^r\right) - \E_r\frac{1}{K}\sum_{k=1}^K \left(g_{i,k-1}^r - g_{i,0}^r\right)\right\|^2\nonumber\\
	\overset{(\ref{eq:sep-variance})}{\le}& (1-p)\left(1+\frac{p}{2}\right)\| g^{r-1} - \nabla f(x^{r-1})\|^2 + \frac{p\sigma^2\sI\{b_1 < M\}}{N b_1}\nonumber\\
	&\quad + (1-p)\left(1 + \frac{2}{p}\right)\left\|\E_r\frac{1}{K|S^r|}\sum_{i\in S^r}\sum_{k=1}^K \left(g_{i,k-1}^r - g_{i,0}^r\right)\right\|^2 \nonumber\\
	&\quad + \frac{1-p}{S}\left(1+ \frac{p}{2}\right)\E_r\left\|\left(\nabla f_i(x^r) - \nabla f(x^r) + \nabla f(x^{r-1}) - \nabla f_i(x^{r-1})\right)\right\|^2 \nonumber\\
	&\quad\quad+ \frac{1-p}{S}\left(1+ \frac{p}{2}\right)\frac{4\sigma^2\sI\{b_2 < M\}}{b_2} \nonumber\\
	&\quad + \frac{1-p}{S}\left(1 + \frac{2}{p}\right) \E_r\left\|\frac{1}{K}\sum_{k=1}^K \left(g_{i,k-1}^r - g_{i,0}^r\right) - \E_r\frac{1}{K}\sum_{k=1}^K \left(g_{i,k-1}^r - g_{i,0}^r\right)\right\|^2 \nonumber\\
	\overset{(\ref{eq:mean-var-ineq})}{\le}& (1-\frac{p}{2})\| g^{r-1} - \nabla f(x^{r-1})\|^2 + \frac{p\sigma^2\sI\{b_1 < M\}}{N b_1} \nonumber\\
	&\quad + \frac{2}{p}\left\|\E_r\frac{1}{K|S^r|}\sum_{i\in S^r}\sum_{k=1}^K \left(g_{i,k-1}^r - g_{i,0}^r\right)\right\|^2 \nonumber\\
	&\quad + \frac{1-p/2}{S}\E_r\left\|\left(\nabla f_i(x^r) - \nabla f_i(x^{r-1})\right)\right\|^2 + \frac{1-p/2}{S}\frac{4\sigma^2\sI\{b_2 < M\}}{b_2} \nonumber\\
	&\quad + \frac{2}{pS} \E_r\left\|\frac{1}{K}\sum_{k=1}^K \left(g_{i,k-1}^r - g_{i,0}^r\right)\right\|^2.\label{eq:graident-error-partial}
	\end{align}
	In the previous derivations, (\ref{eq:seperate-mean-var-1}) comes from the fact that we separate the mean and the variance of a random variable (Equation (\ref{eq:mean-var-separation})). In (\ref{eq:variance}), we define $z_t$ to be $\nabla_{\gI_2} f_i(x^r) - \nabla f(x^r) + \nabla f(x^{r-1}) - \nabla_{\gI_2} f_i(x^{r-1})$ and then apply (\ref{eq:indep-variance}). Here, $z_t$ are i.i.d. random variables.
	
	Then we plug in Lemma~\ref{lem:local-err}, and by setting $\eta_l \le \frac{\sqrt{2}p}{24\sqrt{S}KL}$, we have
	\begin{align}
	&\E_r||g^r - \nabla f(x^r)||^2 \nonumber\\
	\overset{(\ref{eq:graident-error-partial})}{\le}& \left(1-\frac{p}{2}\right)\| g^{r-1} - \nabla f(x^{r-1})\|^2 + \frac{p\sigma^2\sI\{b_1 < M\}}{N b_1}\nonumber\\
	&\quad + \frac{2}{p}\left\|\E_r\frac{1}{K|S^r|}\sum_{i\in S^r}\sum_{k=1}^K \left(g_{i,k-1}^r - g_{i,0}^r\right)\right\|^2 \nonumber\\
	&\quad + \frac{1-p/2}{S}\E_r\|\left(\nabla f_i(x^r) - \nabla f_i(x^{r-1})\right)\|^2 + \frac{1-p/2}{S}\frac{4\sigma^2\sI\{b_2 < M\}}{b_2}\nonumber\\
	&\quad + \frac{2}{pS} \E_r\left\|\frac{1}{K}\sum_{k=1}^K \left(g_{i,k-1}^r - g_{i,0}^r\right)\right\|^2\nonumber\\
	\le& \left(1-\frac{p}{2}\right)\| g^{r-1} - \nabla f(x^{r-1})\|^2 + \frac{p\sigma^2\sI\{b_1 < M\}}{N b_1}\label{eq:plug-in-lemma-local-err}\\
	&\quad + \frac{1-p/2}{S}L^2\E_r\|x^r - x^{r-1}\|^2 + \frac{1-p/2}{S}\frac{4\sigma^2\sI\{b_2 < M\}}{b_2}\nonumber\\
	&\quad + \frac{4}{p} \cdot 12K^2 L^2\eta_l^2\left(\frac{\sigma^2\sI\{b_2 < M\}}{b_2} + L^2\|x^r - x^{r-1}\|^2 + \|g^{r-1} - \nabla f(x^{r-1})\|^2 + \|\nabla f(x^{r-1})\|^2\right)\nonumber\\
	\overset{\text{Plug in }\eta_l}{\le}& (1-\frac{p}{2})\| g^{r-1} - \nabla f(x^{r-1})\|^2 + \frac{p\sigma^2\sI\{b_1 < M\}}{N b_1}\nonumber\\
	&\quad + \frac{1-p/2}{S}L^2\E_r\|x^r - x^{r-1}\|^2 + \frac{1-p/2}{S}\frac{4\sigma^2\sI\{b_2 < M\}}{b_2}\nonumber\\
	&\quad + \frac{p}{6S}\frac{\sigma^2\sI\{b_2 < M\}}{b_2} + \frac{pL^2}{6S}\|x^r - x^{r-1}\|^2 + \frac{p}{6S}\|g^{r-1} - \nabla f(x^{r-1})\|^2 + \frac{48K^2L^2\eta_l^2}{p}\|\nabla f(x^{r-1})\|^2\nonumber\\
	\le& (1-\frac{p}{3})\| g^{r-1} - \nabla f(x^{r-1})\|^2 + \frac{1-p/3}{S}L^2\E_r\|x^r - x^{r-1}\|^2 + \frac{48K^2L^2\eta_l^2}{p}\|\nabla f(x^{r-1})\|^2\nonumber\\
	&\quad + \quad \frac{p}{6S}\frac{\sigma^2\sI\{b_2 < M\}}{b_2} + \frac{1-p/2}{S}\frac{4\sigma^2\sI\{b_2 < M\}}{b_2}+ \frac{p\sigma^2\sI\{b_1 < M\}}{N b_1},\nonumber
	\end{align}
	where in (\ref{eq:plug-in-lemma-local-err}), we apply Lemma~\ref{lem:local-err} and Eq. (\ref{eq:cauchy}).
	Plugging in the batch sizes $b_1 = \min\{M, \frac{24\sigma^2}{pN\epsilon^2}\}, b_2 = \min\{M, \frac{48\sigma^2}{pS\epsilon^2}\}$ and recall $\eta_l \le \frac{\sqrt{2}p}{24\sqrt{S}KL}$, we get
	\begin{align}
	&\E_r||g^r - \nabla f(x^r)||^2 \nonumber\\
	\le& \left(1-\frac{p}{3}\right)\| g^{r-1} - \nabla f(x^{r-1})\|^2 + \frac{1-p/3}{S}L^2\E_r\|x^r - x^{r-1}\|^2 + \frac{48K^2L^2\eta_l^2}{p}\|\nabla f(x^{r-1})\|^2 + \frac{p\epsilon^2}{8} \label{eq:lem-graident-error-middle-step}\\
	\le& (1-\frac{p}{3})\| g^{r-1} - \nabla f(x^{r-1})\|^2 + \frac{1-p/3}{S}L^2\E_r\|x^r - x^{r-1}\|^2 + \frac{p}{6S}\|\nabla f(x^{r-1})\|^2 + \frac{p\epsilon^2}{8}. \nonumber
	\end{align}
\end{proof}

Then, combining with the following descent lemma, we can prove Theorem~\ref{thm:non-convex}.
\begin{lemma}[Lemma 2 in \algname{PAGE}\cite{li2021page}]\label{lem:nonconvex-descent}
	Suppose that $f$ is $L$-smooth and let $x^{t+1}:= x^t-\eta g^t$. Then we have
	\[f(x^{t+1}) \le f(x^t) - \frac{\eta}{2}||\nabla f(x^t)||^2 - \left(\frac{1}{2\eta}-\frac{L}{2}\right)||x^{t+1}-x^{t}||^2+\frac{\eta}{2}||g^t-\nabla f(x^t)||^2.\]
\end{lemma}

\thmnonconvex*

\begin{proof}
	When $\eta_l < \frac{\sqrt{2}p}{24\sqrt{S}KL}$, Lemma~\ref{lem:gradient-error} holds. If we choose the batch sizes $b_1 = \min\{M, \frac{24\sigma^2}{pN\epsilon^2}\}, b_2 = \min\{M, \frac{48\sigma^2}{pS\epsilon^2}\}$, we have
	\begin{align}
	&\E\left[f(x^{r}) - f^* + \frac{3\eta_g}{2p}\|g^r-\nabla f(x^r)\|^2\right]\nonumber\\
	\le& \E\left[f(x^{r-1})-f^*-\frac{\eta_g}{2}\|\nabla f(x^{r-1})\|^2 - \left(\frac{1}{2\eta_g}-\frac{L}{2}\right)\|x^r-x^{r-1}\|^2 + \frac{\eta_g}{2}\|g^{r-1}-\nabla f(x^{r-1})\|^2\right]\label{eq:plug-lemma}\\
	&\quad + \frac{3\eta_g}{2p}\E\left[\left(1-\frac{p}{3}\right)\|\nabla f(x^{r-1})-g^{r-1} \|^2 + \frac{1}{S}\left(1-\frac{p}{3}\right)L^2\E_r\|x^r - x^{r-1}\|^2 + \frac{p}{6S}\E_r\|\nabla f(x^{r-1})\|^2+\frac{p\epsilon^2}{8}\right]\nonumber\\
	\le& \E[f(x^{r-1})-f^*-\frac{\eta_g}{4}\|\nabla f(x^{r-1})\|^2 - \left(\frac{1}{2\eta_g}-\frac{L}{2}-\frac{3\eta_g}{2p}\frac{1}{S}\left(1-\frac{p}{3}\right)L^2\right)\|x^r-x^{r-1}\|^2\nonumber\\
	&\quad + \frac{3\eta_g}{2p}\|g^{r-1}-\nabla f(x^{r-1})\|^2\label{eq:rearrange},
	\end{align}
	where in (\ref{eq:plug-lemma}) we plug in Lemma~\ref{lem:gradient-error} and Lemma~\ref{lem:nonconvex-descent}, and in (\ref{eq:rearrange}) we rearrange the terms.
	
	Choosing $\eta_g = \frac{1}{L\left(1+\sqrt{\frac{3(1-\frac{p}{3})}{2pS}}\right)}$ and $p = \frac{S}{N}$, the coefficient of $\|x^r-x^{r-1}\|^2$ is greater than zero, and we can throw that term away (since $\|x^r - x^{r-1}\| \ge 0$ and the sign is minus). Then we have
	\begin{align*}
	&\E\left[f(x^{r}) - f^* + \frac{3\eta_g}{2p}\|g^r-\nabla f(x^r)\|^2\right]\\
	\le & \E\left[f(x^{r-1})-f^*  + \frac{3\eta_g}{2p}\|g^{r-1}-\nabla f(x^{r-1})\|^2\right]+ \frac{3\eta_g\epsilon^2}{16}-\frac{\eta_g}{4}\E\|\nabla f(x^{r-1})\|^2.
	\end{align*}
	We also know that in the first round,
	\[\E\left\|\frac{1}{N}\sum_{i=1}^N\tilde\nabla_{b_1} f_i(x^0) - \nabla f(x^0)\right\|^2 = \frac{\sigma^2}{Nb_1} \le \frac{p\epsilon^2}{24},\]
	and we have
	\[\E\left[f(x^{r}) - f^* + \frac{3\eta_g}{2p}\|g^r-\nabla f(x^r)\|^2\right] \le \E\left[f(x^{0}) - f^* + \frac{3\eta_g}{2p}\frac{p\epsilon^2}{24}\right] + \frac{3r\eta_g\epsilon^2}{16} - \frac{\eta_g}{4}\sum_{i=0}^r\E\|\nabla f(x^i)\|^2,\]
	which leads to
	\[ \frac{\eta_g}{4}\sum_{i=0}^r\E\|\nabla f(x^i)\|^2 \le \E\left[f(x^{0}) - f^* + \frac{3\eta_g}{2p}\frac{p\epsilon^2}{24}\right] + \frac{3r\eta_g\epsilon^2}{16},\]
	where we use the fact that $\|\cdot\|^2 \ge 0$ and $f(x) - f^* \ge 0$.
	
	So in $O(1 / (\eta_g \epsilon^2))$ number of rounds, \fedpage can find a point $x$ such that $\E\|\nabla f(x)\|^2\le \epsilon^2$, which leads to a point $x$ such that $\E\|\nabla f(x)\|\le \epsilon$. Then since
	\[\frac{1}{\eta_g} = L\left(1 + \sqrt{\frac{3(1-p/3)}{2pS}}\right) = O\left(L\left(1+\frac{\sqrt{N}}{S}\right)\right) = O\left(\frac{\sqrt{N}+S}{S}\right),\]
	we know that \fedpage can find a point $x$ such that $\E\|\nabla f(x)\| \le \epsilon$ in $O\left(\frac{L(\sqrt{N}+S)}{S\epsilon^2}\right)$ number of communication rounds.
\end{proof}

\section{Missing Proof in Section \ref{sec:convex}}
In this section, we show the convergence result of \fedpage in the convex setting. We first show the result when the number of local steps is 1 ($K=1$), where \fedpage reduces to \algname{PAGE} algorithm (Section \ref{sec:page-convex-proof}). Then, we show the result of \fedpage in the convex setting with general number of local steps.

\subsection{Proof of Theorem~\ref{thm:page-convex}}\label{sec:page-convex-proof}
Similar to the notations in the proof for the nonconvex setting, we use $\gF_t$ to denote the filtration when we determine the "gradient" $g^{t-1}$ but not $g^{t}$, i.e. $x^t$ is determined but $x^{t+1}$ is not determined. We use $\E_j[\cdot]$ to denote $\E[\cdot|\gF_j]$.

Recall that in this section, we assume the objective function $f$ is convex and all the functions $\{f_{i}\}$ are averaged $L$-smooth.

\assavgsmooth*

The main difficult to prove Theorem~\ref{thm:page-convex} is that \fedpage\ uses biased gradient estimator, i.e.
\[\mathbb E g^r\neq\mathbb E\nabla f(x^r),\]
for most of the rounds $r$. During the derivation, we will encounter the following inner product term
\[\mathbb E\langle \nabla f(x^{r-1})-g^{r-1}, x^r-x^*\rangle.\]
If the gradient estimator is unbiased, the above inner product is zero and we don't have to worry about this term. But when the gradient estimator is biased, we need to bound this term.

However, since the server using \fedpage\ will communicate with all of the clients with probability $p_r$ in round $r$ to get the full gradient $\nabla f(x^r)$, the following property holds.

\begin{restatable}[]{lemma}{lempagebasicproperty}\label{lem:page-basic-property}
	When the number of local steps is 1 ($K=1$) and we choose the probability $p_r = p$ for all $r$, \fedpage\ satisfies for any $r\ge 1$,
	\[\mathbb E_{r}[g^r - \nabla f(x^r)] = (1-p) (g^{r-1} - \nabla f(x^{r-1})).\]
\end{restatable}

\begin{proof}
    If in round $r$, the server does not communicate with all the client and only communicate with a subset of clients $S^r$, then from the definition of \fedpage, $\Delta y_i^r = -\eta_l g_{i,0}^r$ and we can get
    \[g^r = -\frac{1}{K\eta_l |S^r|}\sum_{i\in S^r}\Delta y_i^t = \frac{1}{|S^r|}\sum_{i\in S^r}g_{i,0}^r.\]
    We use $\E_{\gI}$ to denote the expectation over the minibatch $\gI_2$ to estimate the local full gradient. Then we have
    \begin{align*}
        \E_{r}[g^r - \nabla f(x^r)] 
        =& (1-p)\E_r\left[\frac{1}{|S^r|}\sum_{i\in S^r}g_{i,0}^r - \nabla f(x^r)\right] + p\E_r\left[\frac{1}{N}\sum_{i\in [N]}\tilde\nabla_{b_1} f_i(x^r) - \nabla f(x^r)\right] \\
        =& (1-p)\frac{1}{|S^r|}\E_r\sum_{i\in S^r}\E_{\gI}\left[\left(g_{i,0}^r - \nabla f(x^r)\right)\right] \\
        =& (1-p)\frac{1}{|S^r|}\E_r\sum_{i\in S^r}\E_{\gI}\left[\nabla_{\gI_2} f_i(x^r) - \nabla_{\gI_2} f_i(x^{r-1}) + g^{r-1} - \nabla f(x^r)\right] \\
        =& (1-p)\frac{1}{|S^r|}\E_r\sum_{i\in S^r}\E_{\gI}\left[\nabla_{\gI_2} f_i(x^r) - \nabla_{\gI_2} f_i(x^{r-1}) + g^{r-1} - \nabla f(x^r)\right] \\
        =& (1-p)\frac{1}{|S^r|}\E_r\sum_{i\in S^r}\left[\nabla f_i(x^r) - \nabla f_i(x^{r-1}) + g^{r-1} - \nabla f(x^r)\right] \\
        =& (1-p) (g^{r-1} - \nabla f(x^{r-1})),
    \end{align*}
    where we use the fact that $S^r$ is uniformly chosen from $[N]$ and $\nabla_{\gI_2} f(x)$ is a gradient estimator of $\nabla f(x)$.
\end{proof}

\begin{restatable}[Lemma 3 of \citep{li2021page})]{lemma}{lempagenormsquare}\label{lem:page-norm-square}
	When the number of local steps is 1 and we choose $p_r = p$ for all $r$, \fedpage\ satisfies for any $r\ge 1$,
	\begin{align*}
	\mathbb E_r\|g^r - \nabla f(x^r)\|_2^2 &= (1-p) \|g^{r-1} - \nabla f(x^{r-1})\|_2^2
	+ \frac{1-p}{S}\mathbb E_r\|\nabla f_i(x^r) - \nabla f_i(x^{r-1})\|_2^2.
	\end{align*}
\end{restatable}

Then, we can control the inner product term using the following lemma.

\begin{restatable}{lemma}{lempageinner}\label{lem:page-inner}
	For any $t \ge 2$ and any $c > 0$, we have
	\begin{align*}
	\sum_{r=1}^t\mathbb E\langle\nabla f(x^{r-1})-g^{r-1},x^r-x^*\rangle
	\le& \frac{1}{2p}\mathbb E\sum_{r=1}^t\left(c\|\nabla f(x^{r-1})-g^{r-1}\|_2^2+\frac{1}{c}\|x^r-x^{r-1}\|_2^2\right).
	\end{align*}
\end{restatable}

\begin{proof}
	\begin{align*}
	&\mathbb E\langle\nabla f(x^{r-1})-g^{r-1},x^r-x^*\rangle\\
	=& \mathbb E\langle\nabla f(x^{r-1})-g^{r-1},x^r-x^{r-1}\rangle+\mathbb E\mathbb E_{r-1}\langle\nabla f(x^{r-1})-g^{r-1},x^{r-1}-x^*\rangle\\
	\le& \frac{c}{2}\mathbb E\|f(x^{r-1})-g^{r-1}\|_2^2 +\frac{1}{2c}\mathbb E\|x^r-x^{r-1}\|_2^2+(1-p)\mathbb E\langle\nabla f(x^{r-2})-g^{r-2},x^{r-1}-x^*\rangle,
	\end{align*}
	where the last inequality comes from Young's inequality and Lemma~\ref{lem:page-basic-property}. For $r=1$, we know that $\mathbb E_{r-1}\langle\nabla f(x^{r-1})-g^{r-1},x^{r-1}-x^*\rangle = 0$. Unrolling the inequality recursively, we have
	\begin{align*}
	&\mathbb E\langle\nabla f(x^{r-1})-g^{r-1},x^r-x^*\rangle\\
	\le& \frac{c}{2}\mathbb E\|f(x^{r-1})-g^{r-1}\|_2^2 +\frac{1}{2c}\mathbb E\|x^r-x^{r-1}\|_2^2 +(1-p)\mathbb E\langle\nabla f(x^{r-2})-g^{r-2},x^{r-1}-x^*\rangle\\
	\le& \frac{c}{2}\mathbb E\|f(x^{r-1})-g^{r-1}\|_2^2 +\frac{1}{2c}\mathbb E\|x^r-x^{r-1}\|_2^2\\
	&\quad +\frac{1-p}{2}\mathbb E\left(c\|f(x^{r-2})-g^{r-2}\|_2^2 +\frac{1}{c}\|x^{r-1}-x^{r-2}\|_2^2\right)\\
	&\quad +(1-p)^2\mathbb E\langle\nabla f(x^{r-3})-g^{r-3},x^{r-2}-x^*\rangle \\
	\le& \sum_{r'=1}^r(1-p)^{r-r'}\frac{c}{2}\mathbb E\|f(x^{r'-1})-g^{r'-1}\|_2^2 +\sum_{r'=1}^r(1-p)^{r-r'}\frac{1}{2c}\mathbb E\|x^{r'}-x^{r'-1}\|_2^2.
	\end{align*}
	Then we sum up the inequalities from $r=1$ to $t$, we have
	\begin{align*}
	\sum_{r=1}^t\mathbb E\langle\nabla f(x^{r-1})-g^{r-1},x^r-x^*\rangle
	\le& \frac{1}{2p}\mathbb E\sum_{r=1}^t\left(c\|\nabla f(x^{r-1})-g^{r-1}\|_2^2+\frac{1}{c}\|x^r-x^{r-1}\|_2^2\right).
	\end{align*}
\end{proof}

Given these lemmas, we can now prove Theorem~\ref{thm:convex}. We first prove 2 lemmas related to the function decent of each step, and then show the proof of Theorem~\ref{thm:convex}.

\begin{lemma}\label{lem:descent-1}
	For any $r\ge 0$ and any $\lambda > 0$, we have
	\begin{align*}
	0\le & -\eta_g \E_r[f(x^{r+1})-f(x^*)] + \frac{\eta_g }{2L\lambda}\E_r\|g^r-\nabla f(x^r)\|^2 - \frac{1}{2}\E_r\|x^{r+1}-x^*\|^2 + \frac{1}{2}\|x^r-x^*\|^2\\
	&\quad\left(\frac{\eta_g  L(\lambda+1)}{2}-\frac{1}{2}\right)\E_r\|x^{r+1}-x^r\|^2 + \eta_g (1-p)\langle\nabla f(x^{r-1})-g^{r-1},x^r-x^*\rangle.
	\end{align*}
	
	Here, we define $g^{-1} = \nabla f(x^{-1}) = 0$.
\end{lemma}

\begin{proof}
	For any $r \ge 0$, we have
	\begin{align*}
	&\eta_g (f(x^r)-f(x^*))\\
	\le& \eta_g \langle \nabla f(x^r), x^r - x^*\rangle\\
	=& \eta_g \langle \nabla f(x^r) - (1-p)(\nabla f(x^{r-1}) - g^{r-1}), x^r - x^*\rangle +\eta_g \langle (1-p)(\nabla f(x^{r-1}) - g^{r-1}), x^r - x^*\rangle\\
	=& \eta_g \E_r\langle g^r, x^r-x^*\rangle+\eta_g \langle (1-p)(\nabla f(x^{r-1}) - g^{r-1}), x^r - x^*\rangle\\
	=& \eta_g \E_r\langle g^r, x^r-x^{r+1}\rangle + \eta_g \E_r\langle g^r, x^{r+1}-x^*\rangle +\eta_g \langle (1-p)(\nabla f(x^{r-1}) - g^{r-1}), x^r - x^*\rangle\\
	\le& \eta_g \E_r\langle g^r, x^r-x^{r+1}\rangle - \frac{1}{2}\E_r\|x^{r+1}-x^r\|^2 + \frac{1}{2}\E_r\|x^r-x^*\|^2 - \frac{1}{2}\E_r\|x^{r+1}-x^*\|^2\\
	&\quad +\eta_g \langle (1-p)(\nabla f(x^{r-1}) - g^{r-1}), x^r - x^*\rangle.
	\end{align*}

	We also have
	\begin{align*}
	&\eta_g  \E_r\langle g^r, x^r-x^{r+1}\rangle\\
	=& \eta_g  \E_r\langle g^r - \nabla f(x^r), x^r-x^{r+1}\rangle + \eta_g  \E_r\langle \nabla f(x^r), x^r-x^{r+1}\rangle \\
	\le& \frac{\eta_g }{2\lambda L}\E_r\|g^r-\nabla f(x^r)\|^2 + \frac{\eta_g \lambda L}{2}\E_r\|x^r-x^{r+1}\|^2 + \eta_g (f(x^r)-f(x^{r+1})) + \frac{\eta_g  L}{2}\|x^{r+1}-x^r\|^2.
	\end{align*}
	Summing up the 2 inequalities we conclude the proof.
\end{proof}

\begin{lemma}\label{lem:descent-2}
	For any $r\ge 0$ and any $\lambda > 0$, we have
	\[0\le \eta_g (f(x^r)-f(x^{r+1})) + \frac{\eta_g }{2L \lambda}\|g^r - f(x^r)\|^2 + \left(\frac{\eta_g  L(\lambda+1)}{2}-1\right)\|x^r-x^{r+1}\|^2.\]
\end{lemma}

\begin{proof}
	\begin{align*}
	0 =& \eta_g \langle g^r, x^r-x^{r+1}\rangle + \eta_g \langle g^r, x^{r+1}-x^{r}\rangle \\
	=& \eta_g \langle \nabla f(x^r), x^r-x^{r+1}\rangle + \eta_g \langle g^r - \nabla f(x^r), x^r-x^{r+1}\rangle - \|x^{r+1}-x^{r}\|^2\\
	\le& \eta_g (f(x^r)-f(x^{r+1}))+ \eta_g \langle g^r - \nabla f(x^r), x^r-x^{r+1}\rangle +\left(\frac{\eta_g  L}{2}-1\right) \|x^{r+1}-x^{r}\|^2\\
	\le& \eta_g (f(x^r)-f(x^{r+1}))+ \frac{\eta_g }{2\lambda L}\|g^r - \nabla f(x^r)\|^2 +\left(\frac{\eta_g  L (\lambda + 1)}{2} - 1\right) \|x^{r+1}-x^{r}\|^2.
	\end{align*}
\end{proof}

\thmpageconvex*

\begin{proof}[Proof of Theorem~\ref{thm:page-convex}]
	From Lemma~\ref{lem:descent-1} and Lemma~\ref{lem:descent-2}, for any $\delta > 0$, we have
	\begin{align*}
	&\eta_g \E_r[f(x^{r+1})-f(x^*) + \delta(f(x^r)-f(x^{r+1}))]\\
	\le& \frac{\eta_g (1+\delta)}{2L\lambda}\E_r\|g^r-\nabla f(x^r)\|^2 - \frac{1}{2}\E_r\|x^{r+1}-x^*\|^2 + \frac{1}{2}\|x^r-x^*\|^2\\
	&\quad+\left(\frac{\eta_g  L(\lambda+1)(1+\delta)}{2}-\frac{1+2\delta}{2}\right)\E_r\|x^{r+1}-x^r\|^2 + \eta_g (1-p)\langle\nabla f(x^{r-1})-g^{r-1},x^r-x^*\rangle.
	\end{align*}
	Summing up the inequalities from $r=0$ to $T-1$ and taking the expectation, we have
	\begin{align*}
	&\sum_{r=0}^{T-1}\eta_g \E[f(x^{r+1})-f(x^*)]+\eta_g \delta\E [f(x^T)-f(x^0)]\\
	\le& - \frac{1}{2}\E\|x^{T}-x^*\|^2 + \frac{1}{2}\|x^0-x^*\|^2+ \sum_{r=0}^{T-1}\frac{\eta_g (1+\delta)}{2L\lambda}\E\|g^r-\nabla f(x^r)\|^2 \\
	&\quad+\sum_{r=0}^{T-1}\left((1+\delta)\frac{ L(\lambda+1)\eta_g -1}{2}\E\|x^{r+1}-x^r\|^2 + \eta_g (1-p)\E\langle\nabla f(x^{r-1})-g^{r-1},x^r-x^*\rangle\right)\\
	\le& - \frac{1}{2}\E\|x^{T}-x^*\|^2 + \frac{1}{2}\|x^0-x^*\|^2+ \sum_{r=0}^{T-1}\left(\frac{\eta_g (1+\delta)}{2L\lambda}+\frac{\eta_g (1-p)c}{2p}\right)\E\|g^r-\nabla f(x^r)\|^2 \\
	&\quad+\sum_{r=0}^{T-1}(1+\delta)\left(\frac{ L\eta_g (\lambda+1)-1}{2}+\frac{\eta_g (1-p)}{2cp(1+\delta)}\right)\E\|x^{r+1}-x^r\|^2, 
	\end{align*}
	where we apply Lemma~\ref{lem:page-inner} to bound the inner product term. Then using Lemma~\ref{lem:page-norm-square} and Assumption~\ref{ass:avgsmooth}, we can get the following result.
	\begin{align*}
	\sum_{r=0}^{T-1}\E\|g^r-\nabla f(x^r)\|^2 \le \sum_{r=1}^{T-1}\frac{1-p}{pS}\E\E_r\|\nabla f_i(x^r)-\nabla f_i(x^{r-1})\|^2\le \sum_{r=1}^{T-1}\frac{(1-p)L^2}{pS}\E\|x^r- x^{r-1}\|^2.
	\end{align*}
	Plugging into the previous inequality, we have
	\begin{align*}
	&\sum_{r=0}^{T-1}\eta_g \E[f(x^{r+1})-f(x^*)]+\eta_g \delta\E [f(x^T)-f(x^0)]\\
	\le& - \frac{1}{2}\E\|x^{T}-x^*\|^2 + \frac{1}{2}\|x^0-x^*\|^2+ w\sum_{r=0}^{T-1}\E\|x^{r+1}-x^r\|^2,\\
	\end{align*}
	where
	\[w = (1+\delta)\left(\frac{ L\eta_g (\lambda+1)-1}{2}+\frac{\eta_g (1-p)}{2cp(1+\delta)}\right) +  \left(\frac{\eta_g (1+\delta)}{2L\lambda}+\frac{\eta_g (1-p)c}{2p}\right)\frac{(1-p)L^2}{pS}.\]
	By choosing $\lambda = \sqrt{1/(Sp)}, c = \sqrt{Sp/L^2}$, we have
	\begin{align*}
	w =& \frac{ (1+\delta)L\eta_g (\sqrt{1/(Sp)}+1)}{2} - \frac{1+\delta}{2} + \frac{\eta_g (1-p)}{2p\sqrt{pS / L^2}} +\left(\frac{\eta_g (1+\delta)}{2L\sqrt{1/(Sp)}}+\frac{\eta_g (1-p)\sqrt{Sp / L^2}}{2p}\right)\frac{(1-p)L^2}{pS} \\
	=& (1+\delta)\left(\frac{L\eta_g (\sqrt{1/(Sp)}+1)}{2} - \frac{1}{2} + \frac{\eta_g (1-p)L}{2p^{3/2}\sqrt{S}(1+\delta)} + \left(\frac{\eta_g \sqrt{Sp}}{2L}+\frac{\eta_g (1-p)\sqrt{Sp }}{2pL(1+\delta)}\right)\frac{(1-p)L^2}{pS}\right) \\
	=& (1+\delta)\left(\frac{L\eta_g (\sqrt{1/(Sp)}+1)}{2} - \frac{1}{2} + \frac{\eta_g (1-p)L}{2p^{3/2}\sqrt{S}(1+\delta)} + \left(\frac{\eta_g }{2}+\frac{\eta_g (1-p)}{2p(1+\delta)}\right)\frac{(1-p)L}{\sqrt{pS}}\right) \\
	=& (1+\delta)\left(\underbrace{\frac{L\eta_g (\sqrt{1/(Sp)}+1)}{2}}_{\gA} + \underbrace{\frac{\eta_g (1-p)L}{2p^{3/2}\sqrt{S}(1+\delta)}}_{\gB} + \underbrace{\frac{\eta_g }{2}\frac{(1-p)L}{\sqrt{pS}}}_{\gC}+\underbrace{\frac{\eta_g (1-p)}{2p(1+\delta)}\frac{(1-p)L}{\sqrt{pS}}}_{\gD} - \frac{1}{2}\right).
	\end{align*}
	If $\eta_g  \le O\left(\frac{(1+\delta)p}{L(1+\sqrt{1/(Sp)})}\right)$ and $\delta \le 1/p$, we have
	\begin{align*}
	    \gA =& \frac{L\eta_g (\sqrt{1/(Sp)}+1)}{2} = O\left(L(\sqrt{1/(Sp)}+1)\frac{(1 / p)p}{L(\sqrt{1/(Sp)}+1)}\right) = O(1),\\
	    \gB =& \frac{\eta_g (1-p)L}{2p^{3/2}\sqrt{S}(1+\delta)} = O\left(\frac{(1-p)L}{2p^{3/2}\sqrt{S}}\frac{p}{L(1+\sqrt{1/(Sp)})}\right) = O(1),\\
	    \gC =& \frac{\eta_g }{2}\frac{(1-p)L}{\sqrt{pS}} = O\left(\frac{(1-p)L}{\sqrt{pS}}\frac{(1 / p)p}{L(\sqrt{1/(Sp)}+1)}\right) = O(1),\\
	    \gD =& \frac{(1-p)}{2p}\frac{(1-p)L}{\sqrt{pS}}\frac{p}{L(1+\sqrt{1/(Sp)})} = O(1).
	\end{align*}
	In this way, we can choose $\eta_g $ with a small constant such that $w$ is non-positive, and we can throw that term. In this way,
	\begin{align*}
	\sum_{r=0}^{T-1}\E[f(x^{r+1})-f(x^*)]\le \frac{1}{2\eta_g }\|x^0-x^*\|^2+\delta [f(x^0)-f(x^*)] \le\left(\frac{1}{2\eta_g }+\frac{L\delta}{2}\right)\|x^0-x^*\|^2.
	\end{align*}
	Then we set $\delta = 1/(S^{1/4}p^{3/4})$ and $p = S/N$. We first verify that $\delta \le 1/p$. We have
	\[\delta = \frac{1}{S^{1/4}}\frac{N^{3/4}}{S^{3/4}} = \frac{N^{3/4}}{S} \le \frac{N}{S} = \frac{1}{p}.\]
	Then we choose
	\[\eta_g  = \Theta\left(\frac{(1+\frac{N^{3/4}}{S})\frac{S}{N}}{L(1+\frac{\sqrt{N}}{S})}\right) = \Theta\left(\frac{(S+N^{3/4})\frac{S}{N}}{L(S+\sqrt{N})}\right) = \left\{\begin{aligned}
	\Theta\left(\frac{S}{N^{3/4}L}\right), &\text{ if }S \le \sqrt{N}, \\
	\Theta\left(\frac{1}{N^{1/4}L}\right), &\text{ if }\sqrt{N} < S \le N^{3/4}, \\
	\Theta\left(\frac{S}{NL}\right), &\text{ if } N^{3/4} < S
	\end{aligned}\right..\]
	Then, the number of communication round is bounded by
	\[R = \left\{\begin{aligned}
	O\left(\frac{N^{3/4}L}{S\epsilon}\right), &\text{ if }S \le \sqrt{N}, \\
	O\left(N^{1/4}L/\epsilon\right), &\text{ if }\sqrt{N} < S \le N^{3/4}, \\
	O\left(\frac{NL}{S\epsilon}\right), &\text{ if } N^{3/4} < S
	\end{aligned}\right..\]
\end{proof}

\subsection{Proof of Theorem~\ref{thm:convex}}

The proof idea of Theorem~\ref{thm:convex} is similar to that of Theorem~\ref{thm:page-convex}. The difference between these two proof comes from the fact that in the convex setting with general local steps, the local steps between the communication rounds introduce some local error and we need to take the error into account.

In this section, we assume that all the functions $\{f_{i,j}\}$ are $L$-smooth.

\asssmooth*

Similar to the proof with $K=1$, we first prove 2 lemmas related to the function decent of each step. The following lemma is very similar to Lemma~\ref{lem:descent-1} except the last term, since in the general case, we do not have Lemma~\ref{lem:page-basic-property}.

\begin{lemma}\label{lem:descent-general-1}
    For any $r\ge 0$ and any $\lambda > 0$, we have
    \begin{align*}
        0\le & -\eta_g \E_r[f(x^{r+1})-f(x^*)] + \frac{\eta_g }{2L\lambda}\E_r\|g^r-\nabla f(x^r)\|^2 - \frac{1}{2}\E_r\|x^{r+1}-x^*\|^2 + \frac{1}{2}\|x^r-x^*\|^2\\
        &\quad\left(\frac{\eta_g  L(\lambda+1)}{2}-\frac{1}{2}\right)\E_r\|x^{r+1}-x^r\|^2 + \eta_g  \langle\nabla f(x^{r})-g^{r},x^r-x^*\rangle.
    \end{align*}
    Here, we define $g^{-1} = \nabla f(x^{-1}) = 0$.
\end{lemma}

\begin{proof}
    For any $r \ge 1$, we have
    \begin{align*}
        &\eta_g (f(x^r)-f(x^*))\\
        \le& \eta_g \langle \nabla f(x^r), x^r - x^*\rangle\\
        =& \eta_g \langle \nabla f(x^r) - (\nabla f(x^{r})-g^{r}),x^r-x^*\rangle +\eta_g  \langle\nabla f(x^{r})-g^{r},x^r-x^*\rangle\\
        =& \eta_g \E_r\langle g^r, x^r-x^*\rangle+\eta_g  \langle\nabla f(x^{r})-g^{r},x^r-x^*\rangle\\
        =& \eta_g \E_r\langle g^r, x^r-x^{r+1}\rangle + \eta_g \E_r\langle g^r, x^{r+1}-x^*\rangle +\eta_g  \langle\nabla f(x^{r})-g^{r},x^r-x^*\rangle\\
        \le& \eta_g \E_r\langle g^r, x^r-x^{r+1}\rangle - \frac{1}{2}\E_r\|x^{r+1}-x^r\|^2 + \frac{1}{2}\E_r\|x^r-x^*\|^2 - \frac{1}{2}\E_r\|x^{r+1}-x^*\|^2\\
        &\quad +\eta_g  \langle\nabla f(x^{r})-g^{r},x^r-x^*\rangle.
    \end{align*}
    When $r=0$, the inequality also holds. We also have
    \begin{align*}
        &\eta_g  \E_r\langle g^r, x^r-x^{r+1}\rangle\\
        =& \eta_g  \E_r\langle g^r - \nabla f(x^r), x^r-x^{r+1}\rangle + \eta_g  \E_r\langle \nabla f(x^r), x^r-x^{r+1}\rangle \\
        \le& \frac{\eta_g }{2\lambda L}\E_r\|g^r-\nabla f(x^r)\|^2 + \frac{\eta_g \lambda L}{2}\E_r\|x^r-x^{r+1}\|^2 + \eta_g (f(x^r)-f(x^{r+1})) + \frac{\eta_g  L}{2}\|x^{r+1}-x^r\|^2.
    \end{align*}
    Summing up the two inequalities we conclude the proof.
\end{proof}

\begin{lemma}\label{lem:descent-general-2}
    For any $r\ge 0$ and any $\lambda > 0$, we have
    \[0\le \eta_g (f(x^r)-f(x^{r+1})) + \frac{\eta_g }{2L \lambda}\|g^r - f(x^r)\|^2 + \left(\frac{\eta_g  L(\lambda+1)}{2}-1\right)\|x^r-x^{r+1}\|^2.\]
\end{lemma}

\begin{proof}
    \begin{align*}
        0 =& \eta_g \langle g^r, x^r-x^{r+1}\rangle + \eta_g \langle g^r, x^{r+1}-x^{r}\rangle \\
        =& \eta_g \langle \nabla f(x^r), x^r-x^{r+1}\rangle + \eta_g \langle g^r - \nabla f(x^r), x^r-x^{r+1}\rangle - \|x^{r+1}-x^{r}\|^2\\
        \le& \eta_g (f(x^r)-f(x^{r+1}))+ \eta_g \langle g^r - \nabla f(x^r), x^r-x^{r+1}\rangle +\left(\frac{\eta_g  L}{2}-1\right) \|x^{r+1}-x^{r}\|^2\\
        \le& \eta_g (f(x^r)-f(x^{r+1}))+ \frac{\eta_g }{2\lambda L}\|g^r - \nabla f(x^r)\|^2 +\left(\frac{\eta_g  L (\lambda + 1)}{2} - 1\right) \|x^{r+1}-x^{r}\|^2.
    \end{align*}
\end{proof}

Then we bound the inner product term.

\begin{lemma}\label{lem:inner}
	For any $t \ge 2$ and any $c, c' > 0$, we have
	\begin{align*}
	    &\sum_{r=1}^t \E\langle\nabla f(x^{r})-g^{r},x^r-x^*\rangle \\
	    \le& \frac{1}{2p}\E\sum_{r=1}^t\left(c\|\nabla f(x^{r-1})-g^{r-1}\|_2^2+\frac{1}{c}\|x^r-x^{r-1}\|_2^2 + \frac{c'}{KS}\sum_{k=0}^K\sum_{i\in S^r}\|g_{i,k}^r - g_{i,0}^r\|^2 +\frac{1}{c'}\|x^{r-1} - x^*\|^2\right).
	\end{align*}
\end{lemma}

\begin{proof}
	\begin{align*}
	&\E\langle\nabla f(x^{r})-g^{r},x^r-x^*\rangle\\
	=& \E\langle\nabla f(x^{r})-g^{r},x^r-x^{r-1}\rangle+\E\langle\nabla f(x^{r})-g^{r},x^{r-1}-x^*\rangle\\
	\le& \frac{c}{2}\E\|f(x^{r})-g^{r}\|_2^2 +\frac{1}{2c}\E\|x^r-x^{r-1}\|_2^2+\E\langle\nabla f(x^{r})-g^{r},x^{r-1}-x^*\rangle,
	\end{align*}
	where the last inequality comes from Eq. (\ref{eq:cauchy}). We also have
	\begin{align*}
	    &\E\langle\nabla f(x^{r})-g^{r},x^{r-1}-x^*\rangle\\
	    =& (1-p)\E\langle f(x^{r-1})-g^{r-1},x^{r-1}-x^*\rangle + \E\langle \nabla f(x^r) - g^r - (1-p)(\nabla f(x^{r-1} - g^{r-1}), x^{r-1} - x^*\rangle.
	\end{align*}
	Then we can compute the second term in the previous inequality.
	\begin{align*}
	    &\E\langle \nabla f(x^r) - g^r - (1-p)(\nabla f(x^{r-1}) - g^{r-1}), x^{r-1} - x^*\rangle \\
	    =& (1-p)\E\left\langle \nabla f(x^r) - \frac{1}{KS}\sum_{k=0}^K\sum_{i\in S^r}g_{i,k}^r + \frac{1}{S}\sum_{i\in S^r}g_{i,0}^r - \frac{1}{S}\sum_{i\in S^r}g_{i,0}^r + g^{r-1} - \nabla f(x^{r-1}), x^{r-1} - x^*\right\rangle \\
	    =& (1-p)\E\left\langle - \frac{1}{KS}\sum_{k=0}^K\sum_{i\in S^r}g_{i,k}^r + \frac{1}{S}\sum_{i\in S^r}g_{i,0}^r, x^{r-1} - x^*\right\rangle  + \E\left\langle \nabla f(x^r) - \frac{1}{S}\sum_{i\in S^r}g_{i,0}^r + g^{r-1}, x^{r-1} - x^*\right\rangle \\
	    =& (1-p)\E\left\langle - \frac{1}{KS}\sum_{k=0}^K\sum_{i\in S^r}g_{i,k}^r + \frac{1}{S}\sum_{i\in S^r}g_{i,0}^r, x^{r-1} - x^*\right\rangle \\
	    \le& \frac{c'}{2}\|\frac{1}{KS}\sum_{k=0}^K\sum_{i\in S^r}g_{i,k}^r - \frac{1}{S}\sum_{i\in S^r}g_{i,0}^r\|^2 +\frac{1}{2c'}\|x^{r-1} - x^*\|^2,
	\end{align*}
	where we use the fact that $\E \left[\nabla f(x^r) - \frac{1}{S}\sum_{i\in S^r}g_{i,0}^r + g^{r-1} - \nabla f(x^{r-1})\right] = 0$ and Eq. (\ref{eq:cauchy}).
	
	Combining the computations together, we get for any $c,c' > 0$,
	\begin{align*}
	    &\E\langle\nabla f(x^{r})-g^{r},x^r-x^*\rangle\\
	    \le& (1-p)\E\langle f(x^{r-1})-g^{r-1},x^{r-1}-x^*\rangle + \frac{c}{2}\E\|f(x^{r})-g^{r}\|_2^2 +\frac{1}{2c}\E\|x^r-x^{r-1}\|_2^2 \\
	    &\quad + \frac{c'}{2}\|\frac{1}{KS}\sum_{k=0}^K\sum_{i\in S^r}g_{i,k}^r - \frac{1}{S}\sum_{i\in S^r}g_{i,0}^r\|^2 +\frac{1}{2c'}\|x^{r-1} - x^*\|^2\\
	    \le& (1-p)\E\langle f(x^{r-1})-g^{r-1},x^{r-1}-x^*\rangle + \frac{c}{2}\E\|f(x^{r})-g^{r}\|_2^2 +\frac{1}{2c}\E\|x^r-x^{r-1}\|_2^2 \\
	    &\quad + \frac{c'}{2KS}\sum_{k=0}^K\sum_{i\in S^r}\|g_{i,k}^r - \sum_{i\in S^r}g_{i,0}^r\|^2 +\frac{1}{2c'}\|x^{r-1} - x^*\|^2.
	\end{align*}
	We also know that $\E\langle \nabla f(x^0) - g^0, x^0 - x^*\rangle = 0$, and we can get for any $t \ge 1$,
	\begin{align*}
	    &\sum_{r=1}^t \E\langle\nabla f(x^{r})-g^{r},x^r-x^*\rangle \\
	    \le& \frac{1}{2p}\E\sum_{r=1}^t\left(c\|\nabla f(x^{r-1})-g^{r-1}\|_2^2+\frac{1}{c}\|x^r-x^{r-1}\|_2^2 + \frac{c'}{KS}\sum_{k=0}^K\sum_{i\in S^r}\|g_{i,k}^r - g_{i,0}^r\|^2 +\frac{1}{c'}\|x^{r-1} - x^*\|^2\right).
	\end{align*}
\end{proof}

\begin{lemma}
For any $c,c' > 0$ such that $\frac{t\eta_g }{pc'} \le 1/2$, we have
\begin{align*}
    &\sum_{r=1}^t \E\langle\nabla f(x^{r})-g^{r},x^r-x^*\rangle \\
    \le& \frac{t}{pc'}\|x^0 - x^*\|^2 + \left(\frac{c}{2p} + \frac{t\eta_g  c}{p^2c'}\right)\E\sum_{r=1}^t \|\nabla f(x^{r-1})-g^{r-1}\|_2^2\\
	    &\quad + \left(\frac{1}{2cp} + \frac{tc}{pc'}+ \frac{2\eta_g  t}{p^2c'}\right)\E\sum_{r=1}^t \|x^r-x^{r-1}\|_2^2 + \left(\frac{c'}{2p} + \frac{t\eta_g }{p^2c'}\right)\E\sum_{r=1}^t \frac{1}{KS}\sum_{k=0}^K\sum_{i\in S^r}\|g_{i,k}^r - g_{i,0}^r\|^2.
\end{align*}
\end{lemma}

\begin{proof}
\begin{align*}
    &\|x^r - x^*\|^2\\
    =& \|x^{r-1} - x^* - \eta_g  g^{r-1}\|^2 \\
    =& \|x^{r-1} - x^*\|^2 + \eta_g ^2\|g^{r-1}\|^2 - 2\eta_g \langle g^{r-1}, x^{r-1} - x^*\rangle \\
    =& \|x^{r-1} - x^*\|^2 + \|x^{r} - x^{r-1}\|^2 - 2\eta_g \langle \nabla f(x^{r-1}), x^{r-1} - x^*\rangle - 2\eta_g \langle g^{r-1} - \nabla f(x^{r-1}), x^{r-1} - x^*\rangle \\
    \le& \|x^{r-1} - x^*\|^2 + \|x^{r} - x^{r-1}\|^2 + 2\eta_g  (f(x^*) - f(x^{r-1}) + 2\eta_g \langle  \nabla f(x^{r-1}) - g^{r-1}, x^{r-1} - x^*\rangle \\
    \le& \|x^{r-1} - x^*\|^2 + \|x^{r} - x^{r-1}\|^2 + 2\eta_g \langle  \nabla f(x^{r-1}) - g^{r-1}, x^{r-1} - x^*\rangle \\
    \le& \|x^{0} - x^*\|^2 + \sum_{r'=1}^r\|x^{r'} - x^{r'-1}\|^2 + 2\eta_g \sum_{r'=1}^r\langle  \nabla f(x^{r'-1}) - g^{r'-1}, x^{r'-1} - x^*\rangle.
\end{align*}
For simplicity, we define the following notations,
\begin{align*}
    A_c^t =& \E\sum_{r=1}^t c\|\nabla f(x^{r-1})-g^{r-1}\|_2^2\\
    B_c^t =& \E\sum_{r=1}^t \frac{1}{c}\|x^r-x^{r-1}\|_2^2\\
    C_{c'}^t =& \E\sum_{r=1}^t \frac{c'}{KS}\sum_{k=0}^K\sum_{i\in S^r}\|g_{i,k}^r - g_{i,0}^r\|^2\\
    D^t =& \E\sum_{r=1}^t \|x^{r-1} - x^*\|^2 \\
    D_{c'}^t = & \E\sum_{r=1}^t \frac{1}{c'} \|x^{r-1} - x^*\|^2.
\end{align*}
From Lemma~\ref{lem:inner}, we know that for any $t' \le t$, we have
\begin{align*}
	    \sum_{r=1}^{t'} \E\langle\nabla f(x^{r})-g^{r},x^r-x^*\rangle \le & \frac{1}{2p}\left(A_c^{t'} + B_c^{t'} + C_{c'}^{t'} + D_{c'}^{t'}\right) \\
	    \le & \frac{1}{2p}\left(A_c^{t} + B_c^{t} + C_{c'}^t + D_{c'}^{t}\right),
	\end{align*}
and for any $r \le t$, we can bound $\|x^r - x^*\|^2$ by
\begin{align*}
    & \E\| x^r - x^*\|^2 \\
    \le& \E\|x^{0} - x^*\|^2 + \E\sum_{r'=1}^r\|x^{r'} - x^{r'-1}\|^2 + 2\eta_g \E\sum_{r'=1}^r\langle  \nabla f(x^{r'-1}) - g^{r'-1}, x^{r'-1} - x^*\rangle \\
    \le& \E\|x^{0} - x^*\|^2 + \E\sum_{r'=1}^r\|x^{r'} - x^{r'-1}\|^2 + 2\eta_g \frac{1}{2p}\left(A_c^{r} + B_c^{r} + C_{c'}^r + D_{c'}^{r}\right) \\
    \le& \E\|x^{0} - x^*\|^2 + c B_c^t + \frac{\eta_g }{p}\left(A_c^{t} + B_c^{t} + C_{c'}^t + D_{c'}^{t}\right).
\end{align*}
Then we bound $D^t$, we have
\begin{align*}
    D^t =& \sum_{r=1}^t \|x^{r-1} - x^*\|^2 \\
    \le& \sum_{r=1}^t \left(\E\|x^{0} - x^*\|^2 + c B_c^t + \frac{\eta_g }{p}\left(A_c^{t} + B_c^{t} + C_{c'}^t + D_{c'}^{t}\right)\right) \\
    \le& t\|x^0 - x^*\|^2 + \frac{t\eta_g }{p}A_c^t + t\left(c + \frac{\eta_g }{p}\right)B_c^t + \frac{t\eta_g }{p} C_{c'}^t + \frac{t\eta_g }{pc'}D^t.
\end{align*}
As long as $\frac{t\eta_g }{pc'} \le 1/2$, we have
\begin{align*}
    D^t \le 2t\|x^0 - x^*\|^2 + \frac{2t\eta_g }{p}A_c^t + 2t\left(c + \frac{\eta_g }{p}\right)B_c^t + \frac{2t\eta_g }{p} C_{c'}^t.
\end{align*}
Then we plug this inequality into Lemma~\ref{lem:inner}, we get
\begin{align*}
    &\sum_{r=1}^t \E\langle\nabla f(x^{r})-g^{r},x^r-x^*\rangle \\
	    \le& \frac{1}{2p}\E\sum_{r=1}^t\left(c\|\nabla f(x^{r-1})-g^{r-1}\|_2^2+\frac{1}{c}\|x^r-x^{r-1}\|_2^2 + \frac{c'}{KS}\sum_{k=0}^K\sum_{i\in S^r}\|g_{i,k}^r - g_{i,0}^r\|^2 +\frac{1}{c'}\|x^{r-1} - x^*\|^2\right).\\
	    \le& \frac{1}{2p}\E\sum_{r=1}^t\left(c\|\nabla f(x^{r-1})-g^{r-1}\|_2^2+\frac{1}{c}\|x^r-x^{r-1}\|_2^2 + \frac{c'}{KS}\sum_{k=0}^K\sum_{i\in S^r}\|g_{i,k}^r - g_{i,0}^r\|^2\right) \\
	    &\quad + \frac{1}{2pc'}\left(2t\|x^0 - x^*\|^2 + \frac{2t\eta_g }{p}A_c^t + 2t\left(c + \frac{\eta_g }{p}\right)B_c^t + \frac{2t\eta_g }{p} C_{c'}^t\right)\\
	    \le& \frac{t}{pc'}\|x^0 - x^*\|^2 + \left(\frac{c}{2p} + \frac{t\eta_g  c}{p^2c'}\right)\E\sum_{r=1}^t \|\nabla f(x^{r-1})-g^{r-1}\|_2^2\\
	    &\quad + \left(\frac{1}{2cp} + \frac{t}{pc'}+ \frac{\eta_g  t}{p^2cc'}\right)\E\sum_{r=1}^t \|x^r-x^{r-1}\|_2^2 + \left(\frac{c'}{2p} + \frac{t\eta_g }{p^2}\right)\E\sum_{r=1}^t \frac{1}{KS}\sum_{k=0}^K\sum_{i\in S^r}\|g_{i,k}^r - g_{i,0}^r\|^2.
\end{align*}
\end{proof}

\thmconvex*

\begin{proof}
First we have
\begin{align*}
    & \eta_g \E_r[f(x^{r+1})-f(x^*) + \delta( f(x^{r+1}) - f(x^{r})] \\
    \le&  \frac{\eta_g  (1+\delta)}{2L\lambda}\E_r\|g^r-\nabla f(x^r)\|^2 - \frac{1}{2}\E_r\|x^{r+1}-x^*\|^2 + \frac{1}{2}\|x^r-x^*\|^2\\
        &\quad\left(\frac{\eta_g  L(\lambda+1)(1+\delta)}{2}-\frac{1 + 2\delta}{2}\right)\E_r\|x^{r+1}-x^r\|^2 + \eta_g  \langle\nabla f(x^{r})-g^{r},x^r-x^*\rangle.
\end{align*}
Summing up the inequality and choosing $b_1 = \min\{M,\frac{24\sigma^2\sqrt{S}}{\epsilon p^{3/2}N}\}, b_2 =\min\{M,\frac{48\sigma^2\sqrt{S}}{\epsilon p^{3/2}S}\}$,
\begin{align*}
	&\sum_{r=0}^{T-1}\eta_g \E[f(x^{r+1})-f(x^*)]+\eta_g \delta\E [f(x^T)-f(x^0)]\\
	\le& - \frac{1}{2}\E\|x^{T}-x^*\|^2 + \frac{1}{2}\|x^0-x^*\|^2+ \sum_{r=0}^{T-1}\frac{\eta_g (1+\delta)}{2L\lambda}\E\|g^r-\nabla f(x^r)\|^2 \\
	&\quad+\sum_{r=0}^{T-1}(1+\delta)\frac{ L(\lambda+1)\eta_g -1}{2}\E\|x^{r+1}-x^r\|^2 + \sum_{r=0}^{T-1}\eta_g  \E\langle\nabla f(x^{r-1})-g^{r-1},x^r-x^*\rangle\\
	\le& - \frac{1}{2}\E\|x^{T}-x^*\|^2 + \frac{1}{2}\|x^0-x^*\|^2+ \sum_{r=0}^{T-1}\left(\frac{\eta_g (1+\delta)}{2L\lambda}+\frac{c\eta_g }{2p} + \frac{T\eta_g ^2 c}{p^2c'}\right)\E\|g^r-\nabla f(x^r)\|^2 \\
	&\quad+\sum_{r=0}^{T-1}(1+\delta)\left(\frac{ L\eta_g (\lambda+1)-1}{2}+\frac{\eta_g }{2cp(1+\delta)} + \frac{T\eta_g }{pc'(1+\delta)}+ \frac{\eta_g ^2 T}{p^2cc'(1+\delta)}\right)\E\|x^{r+1}-x^r\|^2\\
	&\quad +\left(\frac{c'\eta_g }{2p} + \frac{T\eta_g ^2}{p^2}\right)\E\sum_{r=0}^{T-1} \frac{1}{KS}\sum_{k=0}^K\sum_{i\in S^r}\|g_{i,k}^r - g_{i,0}^r\|^2 + \frac{T\eta_g }{pc'}\|x^0 - x^*\|^2\\
	\le& - \frac{1}{2}\E\|x^{T}-x^*\|^2 + \frac{1}{2}\|x^0-x^*\|^2+ \frac{T\eta_g }{pc'}\|x^0 - x^*\|^2 + 6K^2L^2\eta_l^2\left(\frac{c'\eta_g }{2p} + \frac{T\eta_g ^2}{p^2}\right)\frac{T\sigma^2\sI\{b_2 < M\}}{b_2}\\
	&\quad + \sum_{r=0}^{T-1}w_1\E\|g^r-\nabla f(x^r)\|^2 \\
	&\quad+\sum_{r=0}^{T-1}(1+\delta)\cdot w_2\cdot \E\|x^{r+1}-x^r\|^2\\
	&\quad + 12K^2L^2\eta_l^2\left(\frac{c'\eta_g }{2p} + \frac{T\eta_g ^2}{p^2}\right)\sum_{r=0}^{T-1}\|\nabla f(x^{r})\|^2,\\
	\overset{(\ref{eq:lem-graident-error-middle-step})}{\le}& - \frac{1}{2}\E\|x^{T}-x^*\|^2 + \frac{1}{2}\|x^0-x^*\|^2+ \frac{T\eta_g }{pc'}\|x^0 - x^*\|^2 + 6K^2L^2\eta_l^2\left(\frac{c'\eta_g }{2p} + \frac{T\eta_g ^2}{p^2}\right)(\sqrt{S}p^{3/2}T\epsilon)\\
	&\quad+\sum_{r=0}^{T-1}\left((1+\delta)\cdot w_2 + \frac{3}{p}\frac{1 - p/3}{S}w_1\right)\cdot \E\|x^{r+1}-x^r\|^2\\
	&\quad + \left(12K^2L^2\eta_l^2\left(\frac{c'\eta_g }{2p} + \frac{T\eta_g^2}{p^2}\right) + \frac{144K^2L^2\eta_l^2}{p^2}w_1\right)\sum_{r=0}^{T-1}\|\nabla f(x^{r})\|^2 + \frac{3\epsilon T p^{3/2}}{8\sqrt{S}}w_1.
	\end{align*}
	where we define
	\begin{align*}
	    w_1 :=& \left(\frac{\eta_g (1+\delta)}{2L\lambda}+\frac{c\eta_g }{2p} + \frac{T\eta_g^2 c}{p^2c'}+12K^2L^2\eta_l^2\left(\frac{c'\eta_g }{2p} + \frac{T\eta_g^2}{p^2}\right)\right), \\
	    w_2 :=& \left(\frac{ L\eta_g (\lambda+1)-1}{2}+\frac{\eta_g }{2cp(1+\delta)} + \frac{T\eta_g }{pc'(1+\delta)}+ \frac{\eta_g^2 T}{p^2cc'(1+\delta)}+\frac{12K^2L^2\eta_l^2}{1+\delta}\left(\frac{c'\eta_g}{2p} + \frac{T\eta_g^2}{p^2}\right)\right).
	\end{align*}
	By choosing $p = \frac{S}{N}, \lambda = \sqrt{1 / Sp}, c = \sqrt{Sp/L^2}, \delta = 1 / (S^{1/4}p^{3/4}), \eta_g = O\left(\frac{(1+\delta)p}{L(1+\sqrt{1/(Sp)})}\right), c' = \frac{2T\eta_g}{p}, \eta_l = O\left(\frac{S}{N^{5/4}KL\sqrt{T}}\right)$, we get
	\begin{align*}
	    \frac{\eta_g(1+\delta)}{2L\lambda} =& (1+\delta) O\left(\frac{(1+\delta)S/N\cdot S}{L^2(1+\sqrt{N}/S)\sqrt{N}}\right) &= (1+\delta)O\left(\frac{S^2}{N^{4/5}L^2}\right),\\
	    \frac{c\eta_g}{p} =& (1+\delta)O\left(\frac{S}{L^2\sqrt{N}(1+\sqrt{N}/S)}\right) &= (1+\delta)O\left(\frac{S^2}{N L^2}\right),\\
	    \frac{T\eta_g^2 c}{p^2c'}=& (1+\delta) O\left(\frac{S}{L^2N(1+\sqrt{N}/S)}\right) &= (1+\delta)O\left(\frac{S^2}{N^{3/2}L^2}\right),\\
	    \frac{c'\eta_g}{2p} + \frac{T\eta_g^2}{p^2} =& 2\frac{T\eta_g^2}{p^2} &= O\left(\frac{2T\sqrt{N}}{L^2}\right).
	\end{align*}
	Then, we can verify that $w_1\cdot N / S^2 = (1+\delta)\cdot O(1)$. Similar to the proof of Theorem~\ref{thm:page-convex}, we can also verify that $w_2 = O(1)$, and we can choose $\eta_g$ and $\eta_l$ with a small constant such that
	\begin{align*}
	    \left((1+\delta)\cdot w_2 + \frac{3}{p}\frac{1 - p/3}{S}w_1\right) \le& 0, \\
	    6K^2L^2\eta_l^2\left(\frac{c'\eta_g}{2p} + \frac{T\eta_g^2}{p^2}\right)(\sqrt{S}p^{3/2}T\epsilon) \le& \frac{\eta_g \epsilon T}{16}, \\
	    \left(12K^2L^2\eta_l^2\left(\frac{c'\eta_g}{2p} + \frac{T\eta_g^2}{p^2}\right) + \frac{144K^2L^2\eta_l^2}{p^2}w_1\right) \le& \frac{\eta_g}{4L}, \\
	    \frac{3\epsilon T p^{3/2}}{8\sqrt{S}}w_1 \le& \frac{3\eta_g \epsilon T}{8}.
	\end{align*}
	Then we have
	\begin{align*}
	    &\sum_{r=0}^{T-1}\E[f(x^{r+1})-f(x^*)] \\
	    \le & \delta\E [f(x^0)-f(x^T)] + \frac{1}{\eta_g}\|x^0-x^*\|^2 + \frac{\epsilon T}{16} + \frac{1}{4L}\sum_{r=0}^{T-1}\|\nabla f(x^{r})\|^2 + \frac{3\epsilon T}{8}\\
	    \le & \delta\E [f(x^0)-f(x^T)] + \frac{1}{\eta_g}\|x^0-x^*\|^2 + \frac{1}{2}\sum_{r=0}^{T-1}\E[f(x^{r})-f(x^*)] + \frac{7\epsilon T}{16}.
	\end{align*}
	Then we know that
	\begin{align*}
	    \frac{1}{T}\sum_{r=0}^{T-1}\E[f(x^{r+1})-f(x^*)] \le& 2\frac{\delta}{T} (f(x^0)-f(x^*)) +\frac{2}{\eta_g T}\|x^0 - x^*\|^2 + \frac{7\epsilon}{16}.
	\end{align*}
	Recall that
	\[\eta_g  = \Theta\left(\frac{(1+\frac{N^{3/4}}{S})\frac{S}{N}}{L(1+\frac{\sqrt{N}}{S})}\right) = \Theta\left(\frac{(S+N^{3/4})\frac{S}{N}}{L(S+\sqrt{N})}\right) = \left\{\begin{aligned}
	\Theta\left(\frac{S}{N^{3/4}L}\right), &\text{ if }S \le \sqrt{N} \\
	\Theta\left(\frac{1}{N^{1/4}L}\right), &\text{ if }\sqrt{N} < S \le N^{3/4} \\
	\Theta\left(\frac{S}{NL}\right), &\text{ if } N^{3/4} < S
	\end{aligned}\right..\]
	We have
	\[\frac{1}{R}\sum_{r=0}^{R-1}\E[f(x^{r+1})-f(x^*)] \le \left\{\begin{aligned}
	O\left(\frac{N^{3/4}L}{SR} + \epsilon\right), &\text{ if }S \le \sqrt{N} \\
	O\left(\frac{N^{1/4}L}{R} + \epsilon\right), &\text{ if }\sqrt{N} < S \le N^{3/4} \\
	O\left(\frac{NL}{SR} + \epsilon\right), &\text{ if } N^{3/4} < S
	\end{aligned}\right..\]
\end{proof}